\documentclass{article}

\PassOptionsToPackage{compress}{natbib}
\usepackage{amsmath}

\usepackage[preprint]{neurips_2021}

\usepackage[utf8]{inputenc} %
\usepackage[T1]{fontenc}    %
\usepackage{url}            %
\usepackage{booktabs}       %
\usepackage{amsfonts}       %
\usepackage{nicefrac}       %
\usepackage{microtype}      %
\usepackage{xcolor}         %
\usepackage{upgreek}
\usepackage{subfigure}

\usepackage{enumerate, amsthm, amssymb, graphicx, pifont, csquotes, dashrule, mathrsfs, tikz, bbm, bm, verbatim, algorithm, algorithmic}
\usepackage{wrapfig}
\usepackage[framemethod=TikZ]{mdframed}
\usepackage[colorlinks]{hyperref}

\definecolor{dblue}{RGB}{98, 140, 190}
\definecolor{dlblue}{RGB}{216, 235, 255}
\definecolor{dgreen}{RGB}{124, 155, 127}
\definecolor{dpink}{RGB}{207, 166, 208}
\definecolor{dyellow}{RGB}{255, 248, 199}
\definecolor{dgray}{RGB}{46, 49, 49}

\newcommand{\durl}[1]{\textcolor{dblue}{\underline{\url{#1}}}}

\newmdenv[
  topline=false,
  bottomline=false,
  rightline = false,
  leftmargin=10pt,
  rightmargin=0pt,
  innertopmargin=0pt,
  innerbottommargin=0pt
]{innerproof}

\newcounter{KDefCounter}

\newcommand{\ddef}[2]
{
\vspace{1mm}
\refstepcounter{KDefCounter}
{\bf Definition \theKDefCounter} (#1): {\it #2}
}

\newtheorem{assumption}{Assumption}

\newtheorem{corollary}{Corollary}

\newtheorem{lemma}{Lemma}

\newtheorem{remark}{Remark}
\newtheorem{theorem}{Theorem}

\usepackage{cancel}
\newcommand{\AF}{{{\mathcal{A} \mathcal{F}}_{{\cal I}^{\rightarrow}}}}
\newcommand{\AFnoarrow}{{{\mathcal{A} \mathcal{F}}_{{\cal I}}}}

\newcommand{\TEI}{{\rm I}^{\rightarrow}_o}
\newcommand{\AFI}{{{\mathcal{A} \mathcal{F}}_{{\cal I}}}}
\newcommand{\TEISet}{{\cal I}^{\rightarrow}}

\usepackage[skins,theorems]{tcolorbox}
\tcbset{highlight math style={enhanced,
  colframe=red,colback=white,arc=0pt,boxrule=1pt}}

\definecolor{friendlygreen}{RGB}{77, 175, 74}
\definecolor{friendlyblue}{RGB}{55, 126, 184}
\definecolor{friendlyred}{RGB}{228, 26, 28}
\hypersetup{               
    linkcolor=friendlyred,
    citecolor=friendlyblue,
}

\newenvironment{enumerate*}%
  {\begin{enumerate}%
    \vskip 0.1in%
    \setlength{\itemsep}{0pt}%
    \setlength{\parskip}{0pt}}%
  {\end{enumerate}%
   \vskip -0.1in}

\title{Temporally Abstract Partial Models}

\author{Khimya Khetarpal \thanks{Correspondence to khimya.khetarpal@mail.mcgill.ca} \textsuperscript{ \rm 1,\rm 2}, Zafarali Ahmed \textsuperscript{\rm 3}, Gheorghe Comanici \textsuperscript{\rm 3}, Doina Precup\textsuperscript{\rm 1,\rm 2,\rm 3} \\   
\textsuperscript{\rm 1}McGill University, \textsuperscript{\rm 2}Mila, $^3$DeepMind}

\begin{document}

\maketitle

\begin{abstract}
 Humans and animals have the ability to reason and make predictions about different courses of action at many time scales. In reinforcement learning, option models (Sutton, Precup \& Singh, 1999; Precup, 2000) provide the framework for this kind of temporally abstract prediction and reasoning. Natural intelligent agents are also able to focus their attention on courses of action that are relevant or feasible in a given situation, sometimes termed affordable actions. In this paper, we define a notion of affordances for options, and develop temporally abstract partial option models, that take into account the fact that an option might be affordable only in certain situations. We analyze the trade-offs between estimation and approximation error in planning and learning when using such models, and identify some interesting special cases. Additionally, we demonstrate empirically the potential impact of partial option models on the efficiency of planning.
\end{abstract}

\section{Introduction}
\label{sec:introduction}
Intelligent agents flexibly reason about the applicability and effects of their actions over different time scales, which in turn allows them to consider different courses of action. Yet modeling the entire complexity of a realistic environment is quite difficult and requires a lot of data~\citep{kakade2003sample}. Animals and people exhibit a powerful ability to control the modelling process by understanding which actions deserve any consideration at all in a situation. By anticipating only certain aspects of their effects over different horizons may make models more predictable or easier to learn. In this paper we develop the theoretical underpinnings of how such an ability could be defined and studied in sequential decision making. We work in the context of model-based reinforcement learning (MBRL)~\citep{sutton2018introduction} and temporal abstraction in the framework of options~\cite{sutton1999between}.   
Theories of embodied cognition and perception suggest that humans are able to represent the world knowledge in the form of {\it internal models} across different time scales~\citep{pezzulo2016navigating}. Option models provide a framework for RL agents to exhibit the same capability. Options define a way of behaving, including a set of states in which an option can start, an internal policy that is used to make decisions while the option is executing, and a stochastic, state-dependent termination condition. Models of options predict the (discounted) reward that an option would receive over time and the (discounted) probability distribution over the states attained at termination~\citep{sutton1999between}. Consequently, option models enable the extension of dynamic programming and many other RL planning methods in order to achieve temporal abstraction, i.e. to be able to consider seamlessly different time scales of decision-making. 

Much of the work on learning and planning with options considers that they apply everywhere~\citep{bacon2017option, harb2017waiting, harutyunyan2019per, harutyunyan2019termination}, with some notable recent exceptions which generalize the notion of initiation sets in the context of function approximation~\citep{khetarpal2020options}. Having options that are partially defined is very important in order to control the complexity of the planning and exploration process. However, the notion of partially defined option models, which  make predictions only from a subset of states, has not yet been explored. This is the focus of our paper.  

In natural intelligence, the ability to make predictions across different scales is linked with the ability to understand the {\it action possibilities} (i.e. affordances)~\citep{gibson1977theory} which arise at the interface of an agent and an environment and are a key component of successful adaptive control~\citep{Fikes1972, Korf1983, Drescher1991, cisek2010neural}. 
Recent work~\citep{khetarpal2020i} has described a way to implement affordances in RL agents, by formalizing a notion of {\em intent} over state space, and then defining an affordance as the set of state-action pairs that {\em achieve} that intent to a certain degree. One can then plan with partial, approximate models that map affordances to intents, incurring a quantifiable amount of error at the benefit of faster learning and deliberation.
In this paper, we generalize the notion of intents and affordances to option models. As we will see in Sec.~\ref{sec:affordances}, this is non-trivial and requires carefully inspecting the definition of option models.
The resulting temporally abstract models are partial, in the sense that they apply only in certain states and options.

\textbf{Key Contributions.} We present a framework defining temporally extended intents, affordances and abstract partial option models (Sec.~\ref{sec:affordances}). We derive theoretical results quantifying the loss incurred when using such models for planning, exposing trade-offs between single-step models and full option models (Sec.~\ref{sec:valueandplanningloss}). Our theoretical guarantees provide insights and decouple the role of affordances from temporal abstraction. Empirically, we demonstrate end-to-end learning of affordances and partial option models, showcasing significant improvement in final performance and sample efficiency when used for planning in the Taxi domain (Sec.~\ref{sec:experiments}). %

\section{Background}
\label{sec:background}
In RL, a decision-making agent interacts with an environment through a sequence of actions, in order to learn a way of behaving (aka policy) that maximizes its value, i.e. long-term expected return~\citep{sutton2018introduction}. This process is typically formalized as a Markov Decision Process (MDP). A finite MDP is a tuple $M=\langle {\cal S}, {\cal A}, r, P, \gamma \rangle $, where ${\cal S}$ is a finite set of states, ${\cal A}$ is a  finite set of actions, $r: {\cal S} \times {\cal A}\rightarrow [0, R_{\max}$] is the reward function, $P:{\cal S} \times {\cal A} \rightarrow Dist({\cal S})$ is the  transition dynamics, mapping state-action pairs to a distribution over next states, and $\gamma \in [0,1)$ is the discount factor. At each time step $t$, the agent observes a state $s_t \in {\cal S}$ and takes an action $a_t \in {\cal A}$ drawn from its policy $\pi : {\cal S} \rightarrow Dist({\cal A})$ and, with probability $P(s_{t+1}|s_t,a_t)$, enters the next state $s_{t+1}\in{\cal S}$ while receiving a numerical reward $r(s_t, a_t)$. The value function of policy $\pi$ in state $s$ is the expectation of the long-term return obtained by executing  $\pi$ from $s$, defined as: $V_\pi(s) = E\left[\sum_{t=0}^{\infty} \gamma^{t} r(S_t, A_t) \big| S_0 = s, A_t \sim \pi(\cdot|S_t), S_{t+1} \sim P(\cdot|S_t, A_t) \; \forall t\right]$.

The goal of the agent is to find an optimal policy, $\pi^*=\arg\max_{\pi} V^{\pi}$. If the model of the MDP, consisting of $r$ and $P$, is given, the value iteration algorithm can be used to obtain the optimal value function, $V^*$, by computing the fixed-point of the Bellman equations~\citep{bellmann1957dynamic}: $V^*(s) = \max_a \Big( r(s,a) + \gamma \sum_{s'} P(s'|s,a) V^*(s') \Big), \forall s.$ The optimal policy $\pi^*$ can be obtained by acting greedily with respect to $V^*$.

\textbf{Semi-Markov Decision Process (SMDP).}
An SMDP~\citep{puterman1994markov} is a generalization of MDPs, in which the amount of time between two decision points is a random variable. The transition model of the environment is therefore a joint distribution over the next decision state and the time, conditioned on the current state and action. SMDPs obey Bellman equations similar to those for MDPs.

\textbf{Options.} Options~\citep{sutton1999between} provide a framework for temporal abstraction which builds on SMDPs, but also leverages the fact that the agent acts in an underlying MDP. A Markovian option $o$ is composed of an \textit{intra-option policy} $\pi_o$, a termination condition $\beta_o: {\cal S} \rightarrow Dist({\cal S})$, where $\beta_o(s)$ is the probability of terminating the option upon entering $s$, and an initiation set $I_o \subseteq {\cal S}$. Let $\Omega$ be the set of all options and $\Omega(s) = \{ o | s \in I_o \}$ denote the set of available options at state $s$. In \textit{call-and-return} option execution, when an agent is at a decision point, it examines its current state $s$, chooses $o \in \Omega(s)$ according to a policy over options $\pi_\Omega(s)$, then follows the internal policy $\pi_o$, until the option terminates according to $\beta_o$. Termination yields a new decision point, where this process is repeated.

\textbf{Option Models.} The model of an option $o$ predicts its reward and transition dynamics following a state $s\in I_o$, as follows: $r(s,o) \doteq  E[ R_{t+1} + \gamma R_{t+2} + \cdots  + \gamma^{k -1} R_{t+k} | S_t=s, O_t=o],$ and $  p(s'|s,o) \doteq \sum_{k=1}^{\infty}  Pr(S_k=s', T_k=1, T_{0<i<k}=0 | S_0=s, A_{0:k-1} \sim \pi_o, T_{0:k-1} \sim \beta_o) =  \sum_{k=1}^{\infty} \gamma^{k} p(s',k|s,o) $, where $T_i$ is an indicator variable equal to $1$ if the option terminates upon entering state $i$, and $0$ otherwise. $p(s',k|s,o)$ is the probability that option $o$ terminates in $s'$ after exactly $k$ time-steps, given that it started at $s$. Bellman optimality equations can then be expressed in terms of option models. The optimal state value function and state-option value function, $V^{*}_\Omega$ and $Q^{*}_O$, are defined as follows:
\begin{equation*}
\label{option-optimal-valuefunctions}
    V^{*}_\Omega(s) = \max_{o \in \Omega(s)}  Q^*(s,o) \mbox{ and }   Q^{*}_\Omega(s,o) =  r(s,o)  + \sum_{s'} p(s'|s,o)  \max_{o' \in \Omega(s')} Q^{*}_O(s',o').
\end{equation*}
\textbf{Partial Models.} MBRL methods build reward and transition models from data, which are then used to plan, e.g. by using the Bellman equations. However, learning an accurate model can be quite difficult, requiring a lot of data. Moreover, the model does not need to be accurate everywhere, as long as it is accurate in relevant places, and/or it provides useful information for identifying good actions. A useful approach is to build {\em partial models}~\citep{talvitie2009simple}, which only make predictions for specific parts of the observation-action space. Partial models come in two flavors: predicting only the outcome of a subset of state-action pairs, or making predictions only about certain parts of the observation space. Option models can be interpreted as partial models, of the first type, because they are defined only on states where the option applies.

\textbf{Affordances.}
\cite{gibson1977theory} coined the term ``affordances'' to describe the fact that certain states enable certain actions, in the context of embodied agents. For instance, a chair ``affords'' sitting for humans, water ``affords'' swimming for fish, etc. As a result, affordances are a function of the environment as well as the agent, and {\it emerge} out of their interaction. In the context of Object Oriented-MDPs~\citep{diuk2008object}, \citet{abel2014toward, abel2015goal} define affordances as propositional functions on states, which assume the existence of \textit{objects} and \textit{object class} descriptions. We build on a more general notion of affordances in MDPs~\citep{khetarpal2020i}, defined as a relation between states and actions, where an action is affordable in a state if its desired outcome (i.e. {\it  intent}) is likely to be achieved.

\section{Affordances for Temporal Abstractions}
\label{sec:affordances}
\begin{figure}[t]
    \centering
    \includegraphics[width=0.8\textwidth]{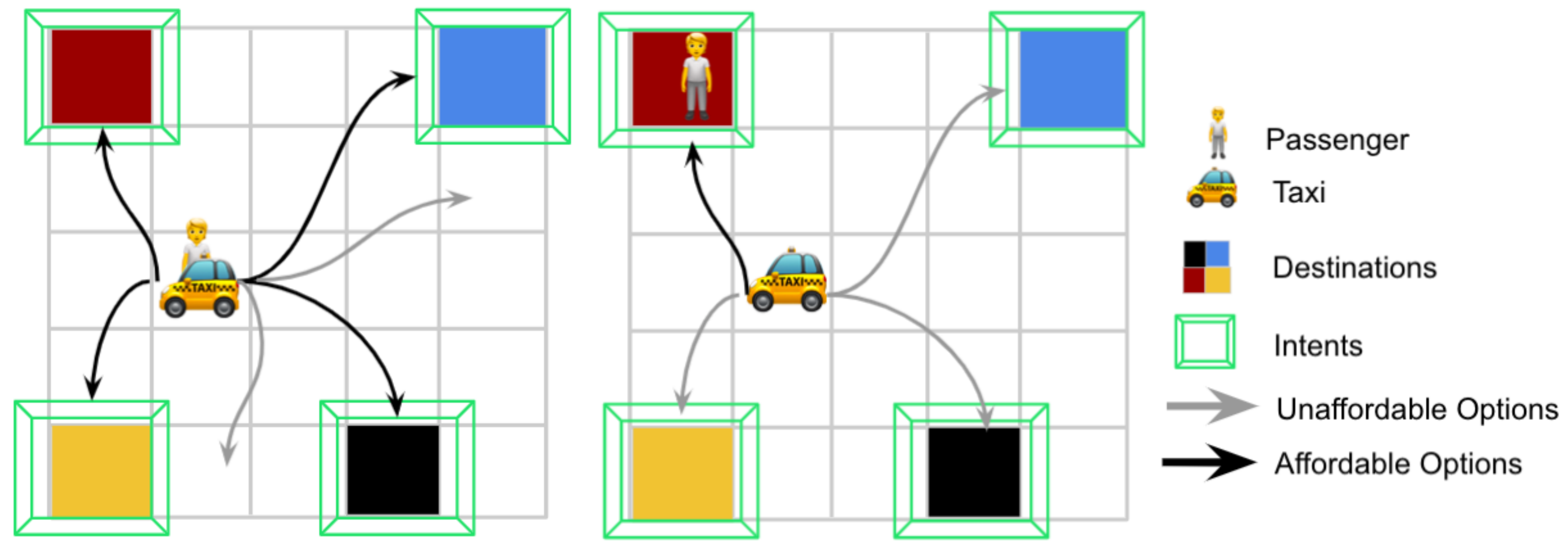}
    \caption{\label{fig:illustration_taxi}\textbf{Illustration:} Intents and affordances in a simple navigation task. Intents include navigation to a particular location to pick up or drop off a passenger. Affordances can indicate e.g. if a passenger can be dropped off (in the case where the passenger is already in the taxi) or if an option to pickup the passenger can succeed or fail (in the case when there is no passenger at the given location). Experiments in this domain are included in Sec.~\ref{sec:experiments}.}
\end{figure}
We seek to reduce both the planning complexity when using option models, and the sample complexity of learning such models, by actively eliminating from consideration choices that are unlikely to improve the planning outcome. In particular, we build temporally abstract partial models informed by affordances. Previous work~\citep{khetarpal2020i} has formalized affordances in RL by considering the desired outcome of a primitive action, i.e. the intent associated with the action. We will now generalize this notion to intents for options, which can be achieved over the duration of the option. 
To make this idea concrete, consider the example of a taxicab, which needs to pick up passengers from given locations and drop them off at a desired destination. As discussed in~\cite{dietterich2000hierarchical}, the use of abstraction, in both state space and time, can help solve this problem. In this context, an option could be to navigate at a particular grid location and an {\em intent} would be to pick up a passenger, or to drop off the passenger currently in the car at the desired destination. Such an intent limits the space of possible options under consideration to those that have desired consequences. These intents capture long-term desired consequences of executing options.

Given the generalization of intents to temporal abstraction, the notion of affordance can still be defined similarly to the primitive action case in ~\cite{khetarpal2020i}, by including state-option pairs which achieve the intent to a certain degree. Indeed, primitive affordances will be a special case of option affordances. Some examples of affordances for our illustration are depicted in Fig.~\ref{fig:illustration_taxi}. %
An agent  can then build partial models of only affordable options enabling it to not only {\it ``navigate in the affordance landscape''}~\citep{pezzulo2016navigating}, but also to better gauge action choices~\citep{cisek2010neural}.

\subsection{Trajectory Based Option Models}

In order to justify the upcoming definitions, we will start with a slight re-writing of the option models in terms of trajectories. %
A trajectory $\tau(t)$ is a random variable, denoting a state-action sequence of length $t\geq 1$, $\tau(t)=\langle S_0, A_0, \dots S_{t-1}, A_{t-1}, S_t\rangle$. 
Overloading notation, let $\tau(s,t)$ denote a trajectory of length $t$ for which $S_0=s$. Further, let $\tau(s,t,s')$ be a trajectory of length $t$ with $S_0=s$ and $S_t=s'$ and $\tau(s,s')$ a trajectory of {\em any length} $t$ for which $S_0=s$ and $S_t=s'$. 
The return is then a deterministic function of a trajectory:
$G(\tau)=\sum_{k=0}^{|\tau|-1} \gamma^k r(S_k, A_k)$, where $|\tau|$ is the length of the trajectory.
The probability of observing a given trajectory $\langle s, a_0 \dots s_t\rangle$, $s\in I_o$, under option $o$ is: 
\medmuskip=0mu \thinmuskip=0mu \thickmuskip=0mu
{\small\[
P(\tau=\langle s_0, a_0 \dots s_t\rangle|o) = \left(\prod_{k=0}^{t-1} \pi_o(A_k=a_k|S_k=s_k)  P(S_{k+1}=s_{k+1}|S_k=s_k,A_k=a_k) (1-\beta_o(s_{k+1}))\right) \frac{\beta_o(s_t)}{1-\beta_o(s_t)}
\]}
\medmuskip=4mu \thinmuskip=4mu \thickmuskip=4mu
where the last fraction is there just to capture correctly termination at $t$. To simplify notation, we denote this by $P_o(\tau(s,t))$. We can define analogously the probability of a trajectory being generated by $o$ starting at state $s\in I_o$ and ending at a given state $s'$ after $t$ steps by $P_o(\tau(s,t,s'))$.
The probability of a trajectory of any length $\tau(s,s')$ under $o$ is then: 
$
P_o(\tau(s,s'))=\sum_{t=1}^\infty P_o(\tau(s,t,s'))
$ 
Let ${\cal T}(s,t,s')$ denote the set of all trajectories starting at $s$, ending at $s'$ and of length $t$ and ${\cal T}(s,s')=\cup_t {\cal T}(s,t,s')$.
We can write the undiscounted transition model of an option $o$ as:
\[
P(s'|s,o) = \textstyle \sum_{\tau(s,s') \in {\cal T}(s,s')} P_o(\tau(s,s'))
\]
The discount on a trajectory $\tau$ will be denoted  $\gamma(\tau)$. If the discount factor is fixed per time step, this will simply be $\gamma^{|\tau|}$; 
all trajectories of the same length will have the same discount, which will allow us to factor it out of products.

The reward model of an option is:
\begin{align*}
    r(s,o,s') = \textstyle  \sum_{t=1}^{\infty} \sum_{\tau(s,t,s')\in {\cal T}(s,t,s')}P_{o}(\tau(s,t,s')) G(\tau(s,t,s'))
\end{align*}
The expected discount for option $o$ on a trajectory going from $s$ to $s'$ is defined as:
\[
\gamma_o(s,s') = \textstyle  \sum_{t=1}^{\infty} \sum_{\tau(s,t,s')\in {\cal T}(s,t,s')}P_o(\tau(s,t,s')) \gamma(\tau(s,t,s'))
\]
Note that when the action is a primitive action, then $\gamma_o(s,s') = \sum_{s'} P(s'|s,o) \gamma$
We can re-write the optimal value function of an option as:
\[
Q^*(s,o) = \sum_{s' \in \mathcal{S}} \sum_{t=1}^\infty\sum_{\tau(s,t,s')} P(\tau(s,t,s')|o)[ G(\tau(s,t,s')) + \gamma(\tau(s,t,s')) \max_{o'} Q^*(s', o')]
\]
Note that the order of the two outer sums can be reversed.
This form is equivalent to the one in~\cite{sutton1999between}, but will be more useful for our results.

\subsection{Option Affordances}

We will now define an intent through a desired probability distribution in the space of all possible trajectories of an option. 
The goal will be to obtain a strict generalization of the results established in~\cite{khetarpal2020i} for primitive actions, in the case where each action is an option and $\beta(s)=1, \forall s$.

\ddef{Temporally Extended Intent $\TEI$}{  A temporally extended intent of option $o\in\Omega$, $\TEI: \mathcal{S} \to \text{Dist}(\mathcal{T})$ specifies for each state $s$ a probability distribution over the space of trajectories ${\cal T}$, describing the \textit{intended} result of executing $o$ in $s$. The associated intent model will be denoted by $P_I( \tau | s,o )= \TEI(s, \tau)$. A temporally extended intent $\TEI$ is satisfied to a degree, $\zeta_{s,o}$ at state $s\in \cal \mathcal{S}$ and option $o \in \Omega$ if and only if:
\begin{equation}
    \label{eq:generalintent_option}
   d(P_I(\tau| s, o), P_o(\tau(s)))\leq \zeta_{s,o},
\end{equation}
where $d$ is a metric between probability distributions\footnote{In this work, we use $d$ to be the total variation.}, and $\tau(s)$ denotes the trajectory starting in state $s$ and following the option $o$.
{\label{definitionintentgeneral}}
}

We note that primitive actions have a ``degenerate" trajectory, consisting of only the next state. Hence, the only reasonable choice there is to define intent based on the next-state distribution, as done in~\cite{khetarpal2020i}. However, options have a whole trajectory, and defining intents on the trajectory distribution provides maximum flexibility. In practice, we expect that most useful intents would be defined in relation with the endpoint of the option, e.g. specifying an intended distribution over the state at the end of the option, or over the joint distribution of the state and duration. Further discussion of special cases is included in the Appendix. Based on this notion of temporally extended intents, {\em  affordances} for options can be defined as follows: 

\ddef{Option Affordances $\AF$}{Given a set of options $\mathcal{O} \subseteq \Omega$ and set of temporally extended intents ${\TEISet}=\cup_{o \in \mathcal{O}} \TEI$, and $\zeta^{\mathcal{I}^{\rightarrow}} \in [0,1]$, we define the affordances $\AF$ associated with $\TEISet$ as a relation $\AF \subseteq {\cal S} \times \mathcal{O}$, such that \label{def:optionaffordances}} 
$\forall (s,o) \in \AF, \TEI \mbox{ is satisfied to at } (s, o) \mbox{ to degree } \zeta_{s,o} \leq \zeta^{\mathcal{I}^{\rightarrow}}$.

Intuitively, we specify temporally extended intents such as ``pick up passenger'', ``drop a passenger at destination'', etc. such that the intent is satisfied to a certain degree. Affordances can then be defined as the subset of state-option pairs that can satisfy the intent to a that degree. Fig.~\ref{fig:illustration_taxi} depicts a cartoon illustration of intents and corresponding option affordances in the classic Taxi environment.

\section{Theoretical Analysis}
\label{sec:valueandplanningloss}
We now analyze the value loss (Sec.~\ref{sec:valueloss}) and planning loss (Sec.~\ref{sec:planninglossboundproof}) induced by {\em temporally extended intents} $\TEISet$ and corresponding {\em temporally abstract affordances} $\AF$. 

\begin{lemma}
\label{lemma1}
Given  a finite set of option $\mathcal{O} \subset \Omega$ and a set of temporally extended intents ${\TEISet}=\cup_{o \in {\cal O}} \TEI$ that are satisfied to degrees $\zeta_{s,o}$, there exist constants ($\zeta^{\mathcal{I}^{\rightarrow}}_P, \zeta^{\mathcal{I}^{\rightarrow}}_R$), such that:
\label{ass:generalzetaconstant}
\begin{align}
     \textstyle \max_{s, o,t, s'} \sum_{\tau(s,t,s')\in {\cal T}(s,t,s')} \Big| P_o(\tau(s,t,s')) -  P_I(\tau(s,t,s')|s, o)) \Big| & \leq \zeta^{\mathcal{I}^{\rightarrow}}_P \mbox{ and } \\
    \textstyle  \max_{s, o}  \Big| r(s,o) - E_{\tau \sim P_I} [G(\tau | s,o )] \Big|  & \leq \zeta^{\mathcal{I}^{\rightarrow}}_R
\end{align}
where $\zeta^{\mathcal{I}^{\rightarrow}}_P = \max_{s,o} \zeta_{s,o}$, $\zeta^{\mathcal{I}^{\rightarrow}}_R = \zeta^{\mathcal{I}^{\rightarrow}}_P || G ||_\infty$, and $G(\tau)$ is the accumulated return on the trajectory $\tau$.
\end{lemma}

The proof is in the Appendix~\ref{proof-lemma1}. We note that the error in the approximate probability distribution is bounded by the degree of intent satisfaction for each option i.e $\zeta_{s,o}$. If intents are far from the true distribution $P$ (i.e. much larger $d$ in Def.~\ref{definitionintentgeneral}) or misspecified, then the bounds above are predominantly governed by the approximation error induced due to the intent specification. Moreover, the approximate reward distribution is also a factor of the error in approximating probability distribution.

\subsection{Value Loss Bound}
\label{sec:valueloss}

A set of {\em temporally extended intents} $\TEISet$ define an intent-induced SMDP $\mathcal{M}_{\TEISet}$, in which the intents can be used to approximate the option transition and reward models. The lemma above establishes this approximation, 
which in turn allows us to compute the value loss incurred when planning in the intent-induced SMDP.

\begin{theorem}[Trajectory-Based Value-Loss Bound]
\label{thm:trajectories_value_loss_analysis}
Given a SMDP $\mathcal{M}$ corresponding to a finite set of options $\mathcal{O}$ and a set of temporally extended intents ${\TEISet}=\cup_{o \in {\cal O}} \TEI$ defined on option trajectories (Def.~\ref{definitionintentgeneral}), the value loss between the optimal policy for the original SMDP $\cal{M}$ and the optimal policy $\pi^*_{\TEISet}$ for the induced SMDP $\cal{M}_{\TEISet}$ is given by:
\begin{align}
\label{eq:trajectories_value_loss_analysis}
    \Big| \Big| V^{\pi^{*}_{\TEISet}} - V^* \Big| \Big|_{\infty}  &\leq \frac{\zeta^{\mathcal{I}^{\rightarrow}}_R}{\Big( 1-\upgamma^{\TEISet} \Big)} + \frac{2 R^{\mathcal{O}}_{max} \max_{s,o} \sum_{t=1}^\infty \gamma^t  |\mathcal{S}| \zeta^{\mathcal{I}^{\rightarrow}}_P %
    }{\Big( 1-\upgamma^{\TEISet} \Big)\Big( 1-\upgamma^{\mathcal{O}} \Big)} 
\end{align}
where $\zeta^{\mathcal{I}^{\rightarrow}}_P$ and $\zeta^{\mathcal{I}^{\rightarrow}}_R$ are defined in Lemma 1, $R^{\mathcal{O}}_{max}=\max_{s,o} r(s,o)$ is the maximum option reward, $\upgamma^{\TEISet} = \max_{s,o} \sum_{s'} \gamma_o^I(s,s')$ and $\upgamma^{\cal O} = \max_{s,o} \sum_{s'} \gamma_o(s,s')$ are the maximum expected discount factor for the intents and options respectively.
\end{theorem}

Proof is in Appendix~\ref{sec:app-trajectories_value_loss_analysis}. 
Our result is a strict generalization of the results established for primitive actions~\citep{khetarpal2020i}. %
Note that the value loss bound is better for temporally extended options than for primitives, due to the dependence on the maximum expected option discount (See Table~\ref{tab:valueloss-comparison}). Note that in our bounds, $R^{\mathcal{O}}_{max}$ and $R_{max}$ denote the maximum achievable reward for options and primitive actions respectively. Further interesting corollaries are included in the Appendix.
\begin{table}[h]
\centering
\begin{tabular}{l|l|l|}
&\multicolumn{2}{c|}{\textbf{Value Loss Bound}} \\ \hline
\textbf{Actions} & \textbf{Sub-probability Intent } & \textbf{Trajectory based Intent}\\ \hline
\textbf{Primitive}   & $2\zeta^\mathcal{I} \frac{\gamma R_{max}}{(1-\gamma)^{2}}$ & - \\
\textbf{Temporally Extended}   &  $ 2\zeta^{\mathcal{I}^{\rightarrow}} \frac{ \upgamma R^{\mathcal{O}}_{max}}{(1-\upgamma)^2}$  &  \color{friendlyblue} $\frac{\zeta^{\mathcal{I}^{\rightarrow}}_R}{\Big( 1-\upgamma^{\TEISet} \Big)} + \frac{2 R^{\mathcal{O}}_{max} \max_{s,o} \sum_{t=1}^\infty \gamma^t  |\mathcal{S}|\zeta^{\mathcal{I}^{\rightarrow}}_P}{\Big( 1-\upgamma^{\TEISet} \Big)\Big( 1-\upgamma^{\cal O} \Big)}$  \\
\end{tabular}
\caption{\textbf{Value Loss Analysis.} The maximum value loss incurred when considering intents shows that while both primitive ($\mathcal{I}$) and temporally extended intents ($\TEISet$) predominantly depend on the intent approximation error $\zeta$, temporally extended intents can result in gains contingent on the closeness of the intent model and maximum expected discounting of options and intents.} %
\label{tab:valueloss-comparison}
\vskip -0.1in
\end{table}

\subsection{Planning Loss Bound}
\label{sec:planninglossboundproof}
In this section, we analyze the effect of incorporating affordances and use temporally extended intents to build partial option models from data on the speed of planning. Similar results have  previously been established to spell out the role of the planning horizon~\citep{jiang2015dependence} and to plan affordance-based partial models of primitive actions~\citep{khetarpal2020i}.

In practical scenarios, the agent may have limited information about the true model of the world. Moreover, it might be infeasible and intractable to build a full model, especially in real-life applications. To address this, we consider the SMDP $\mathcal{M}_{\TEISet}$ induced by models associated with temporally extended intents and the associated affordances, and quantify the loss incurred when planning with this model. 
\begin{theorem}[Trajectory-Based Planning-Loss Bound]
Let $\TEISet$ be a set of temporally extended intents for a finite set of options $\mathcal{O}$, and $\hat{M}_\AF$ the corresponding approximate SMDP over affordable state-option pairs $\AF$. Then, the loss incurred when using  $\hat{M}_\AF$ to compute a policy $\pi^*_{\hat{\cal M}_\AF}$ and then using this policy in the original MDP ${\cal M}$ (also known as the certainty-equivalence planning loss) can be bounded by:
\medmuskip=0mu \thinmuskip=0mu \thickmuskip=0mu
$$
   \Big| \Big|  V^* - V^{\pi^*_{\hat{\cal M}_\AF}} \Big| \Big|_{\infty} \leq \frac{5 \zeta^{\mathcal{I}^{\rightarrow}}_R}{\left( 1-\upgamma^{\TEISet} \right)}  + \frac{2 R^{\mathcal{O}}_{max}}{\left( 1-\upgamma^{\TEISet} \right) \left( 1-\upgamma^{\cal O} \right)} \Big( 2 \max_{s,o} \sum_{t=1}^\infty \gamma^t  |\mathcal{S}| \zeta^{\mathcal{I}^{\rightarrow}}_P + \sqrt{\frac{1}{2n} \log \frac{2 |\AF| |\Pi_{\TEISet}|}{\delta}} \Big)
$$
\medmuskip=4mu \thinmuskip=4mu \thickmuskip=4mu
with probability at least $1-\delta$, where $\zeta^{\mathcal{I}^{\rightarrow}}_P$ and $\zeta^{\mathcal{I}^{\rightarrow}}_R$ are defined in Lemma~\ref{lemma1}, $R^{\mathcal{O}}_{max}=\max_{s,o} r(s,o)$ is the maximum option reward, $\upgamma^{\TEISet} = \max_{s,o} \sum_{s'} \gamma_o^I(s,s')$ and $\upgamma^{\cal O}= \max_{s,o} \sum_{s'} \gamma_o(s,s')$ are the maximum expected discount factor for the intents and options respectively.
\label{theorem:trajectories_planningvaluelossbound}
\end{theorem}
The proof is in Appendix~\ref{sec:app-trajectories_planninglossboundproof}. The planning loss result generalizes the result for primitive actions provided in ~\citet{khetarpal2020i}. We note a similar effect of incorporating affordances in partial models for temporally extended actions. The accuracy in approximation of the intent (via ($\zeta^{\mathcal{I}^{\rightarrow}}_P, \zeta^{\mathcal{I}^{\rightarrow}}_R$)) the size of affordable state-option pairs $|\AF|$, and the SMDP policy class 
$\Pi_\TEISet$ will induce a trade-off between approximation of the intents and space of affordances. A key difference in planning with the approximate partial option models $\hat{M}_\AF$ is that the error can be controlled through the maximum expected discount factor for both intent and option models which in turn depends on the minimum expected duration of all affordable options.

\begin{table*}[t]
\centering
\label{tab:param-OC}
\begin{tabular}{l|l|l|}
&\multicolumn{2}{c|}{\textbf{Planning Loss Bound}} \\ \hline
\textbf{Actions} & \textbf{Without Affordances} & \textbf{Affordance-aware} \\ \hline
\textbf{Primitive}   & $\frac{2 R_{max}}{(1-\gamma)^2} \times\Bigg( \sqrt{\frac{1}{2n} \log \frac{2 |S| |A| |\Pi_{S \times A}|}{\delta}}\Bigg)$    & $\frac{2 R_{max}}{(1-\gamma)^2} \times\Bigg(2\gamma \zeta^\mathcal{I} + \sqrt{\frac{1}{2n} \log \frac{2 |\AFI| |\Pi_\mathcal{I}|}{\delta}}\Bigg)$        \\

\textbf{TEA}   & $\frac{2 R^{\mathcal{O}}_{max}}{(1-\upgamma)^2} \left( \sqrt{\frac{1}{2n} \log \frac{2 |S| |\mathcal{O}| |\Pi_{S \times \mathcal{O}}| }{\delta}} \right)$  & \color{friendlyblue}$\frac{2 R^{\mathcal{O}}_{max}}{(1-\upgamma)^2} \left( 2 \upgamma \zeta^{\mathcal{I}^{\rightarrow}} +  \sqrt{\frac{1}{2n} \log \frac{2 |\AF| |\Pi_{\mathcal{I}^{\rightarrow}}|}{\delta}} \right)$    \\
\end{tabular}
\caption{\label{tab:planningloss-comparison}\textbf{On the role of affordances in actions and options.} We decouple the role of the temporal extent of the options and the effects of incorporating affordances. Our analysis establishes improved guarantees for planning with option models. Further gains are obtained when affordances are incorporated, though at the cost of increased approximation error due to intents through $\zeta$. We note that for simplicity, we present the bounds obtained when intents are defined on the distribution of an option's terminal state, a corollary of Theorem 2. The table highlights the trade-offs between \emph{estimation} (via the model learning depending on the data size $n$) and \emph{approximation} (via the specification of intents).}
\end{table*}

Table~\ref{tab:planningloss-comparison} summarizes the effects of using temporally extended models and affordances. First, we note that the planning with affordances introduces a trade-off between \emph{approximation} and \emph{estimation} in both primitive and temporally extended actions. Concretely, the approximation error is induced due to the specification of intents through $\zeta^{\mathcal{I}^{\rightarrow}}$, whereas the estimation error is induced due to learning of the transition and has a dependence on the data size $n$ and the size of the policy class $\Pi_{\TEISet}$.

\section{Empirical Analysis}
\label{sec:experiments}
In this section, we study the impact of using affordances to learn partial option models which are then used for planning, in order to corroborate the theoretical results established in Sec.~\ref{sec:valueandplanningloss}. In Sec.~\ref{sec:experiment1}, we use a hand designed set of affordances to show that it can improve training stability as well as sample efficiency when used to learn a single partial option model, conditioned on a state-option pair. Then, in Sec.~\ref{sec:experiment2} we demonstrate the viability of learning the set of affordances at the same time as the partial option model resulting in a set of affordances that were smaller than those that were hand designed. 
\begin{wrapfigure}{r}{0.4\textwidth}
\vskip -0.1in
\includegraphics[width=0.42\textwidth]{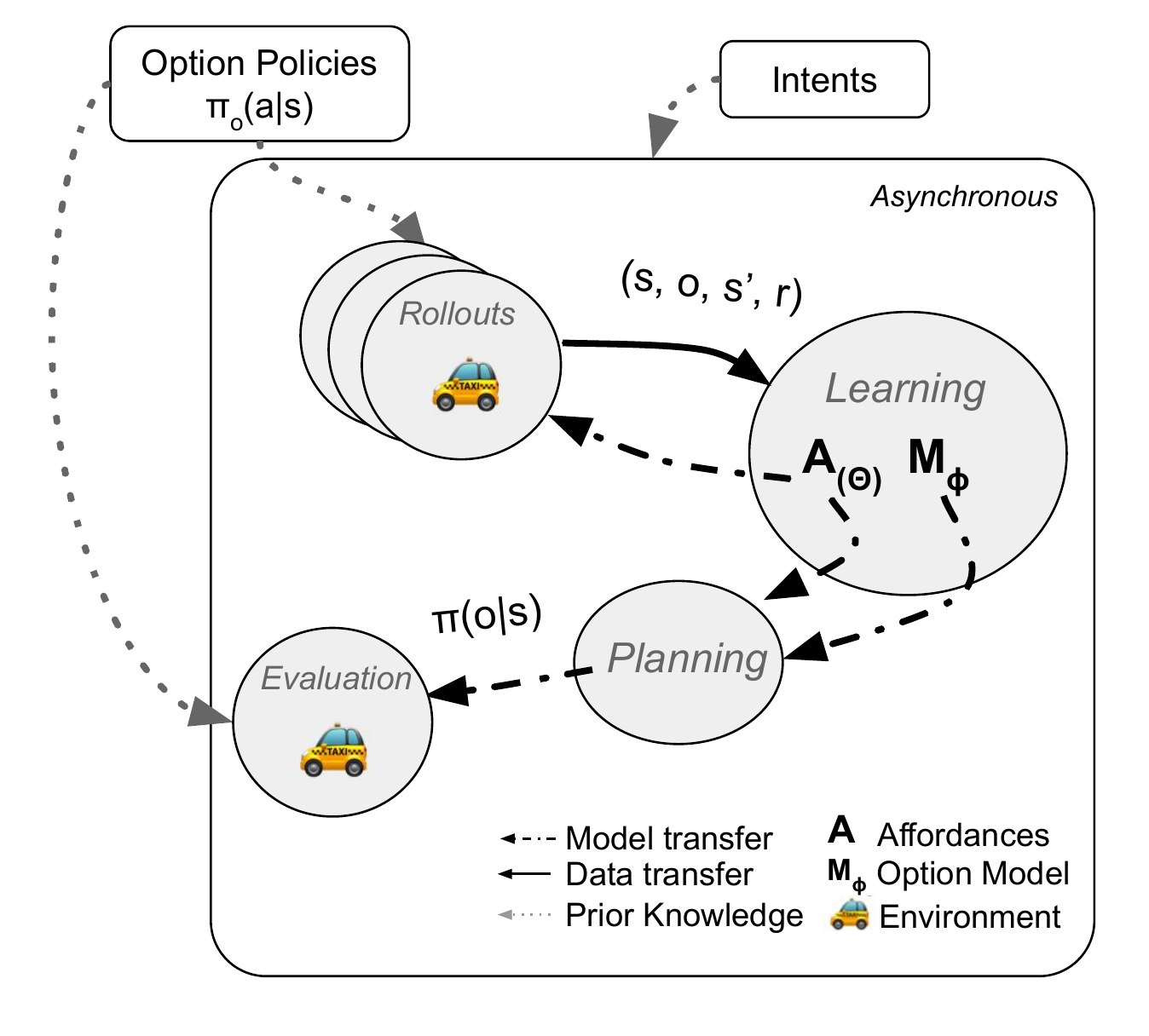}
\caption{\label{fig:pipeline}\textbf{Experimental pipeline.}}
\vskip -0.1in
\end{wrapfigure}
{\bf Environment.} We consider the $5 \times 5$ Taxi domain \citep{dietterich2000hierarchical}. The domain is a grid world with four designated pickup/drop locations, marked as R(ed), B(lue), G(reen), and Y(ellow). See Fig.~\ref{fig:illustration_taxi} for illustration. The agent controls a taxi and faces an episodic problem: the taxi starts in a randomly-chosen square and is given a goal location at which a passenger must be dropped. The passenger is at one of the three other locations. To complete the task, the agent must drive the taxi to the passenger's location, pick them up, go to the destination, and drop the passenger there. The action space consists of six primitive actions: Up, Down, Left, Right, Pickup, and Drop. The agent gets a reward of $-1$ per step, $+20$ for successfully dropping the passenger at the goal and $-10$ for dropping the passenger at the wrong location. There are a total of $25$ (grid positions) $\times 4$  (goal destinations) $\times 5$  (passenger scenarios) $= 500$ states in this environment and the observation is a one-hot vector. 

{\bf Option set $\mathcal{O}$.} We consider a fixed set of \emph{taxi-centric} options, defined as follows: Go to a grid position (25 options); Drop passenger at grid position (25 options); Pickup passenger from grid position (25 options). The options are pre-trained via value iteration for all our experiments. In total there are $75\times500=37500$ state-option pairs.

{\bf Experimental pipeline.
}
We use pre-trained options, $o = \langle I_o = \mathcal{S}, \pi_o(a|s), \beta_o(s) \rangle$, to collect transition data $(s_t, o, T, s_{t+T}, r=\sum_{{i=t}}^Tr_{i})$ where option $o$ was initiated at state $s_t$ and ended in state $s_{t+T}$ after $T$ steps, accumulating a reward of $r$. We execute options until termination or for $T_\text{max}$ steps, whichever comes first. We learn linear models to predict the next state distribution $\hat{P}_{\phi_1}(s' | s, o)$, option duration,  $\hat{L}_{\phi_2}(s,o)$ and reward $\hat{r}_{\phi_3}(s, o)$, where $\phi$ denote parameter vectors. Affordances can be incorporated in model learning by selecting only affordable options during the data collection and to mask the loss of unaffordable state-option transitions:
\begin{equation}
    \label{eqn:masked_model_loss}
    \textstyle \sum_{(s, o, T, s', r)\in\mathcal{D}}A(s,o,s',\TEISet)\big[-\log\hat{P}_{\phi_1}(s'|o,s) + (\hat{L}_{\phi_2}(o,s) - T)^2 + (\hat{r}_{\phi_3}(s,o) - r)^2\big]
\end{equation}
where $A(s,o,s',I)$ is 1 if $(s,o,s')$ is affordable according to the intent $I$ and $0$ otherwise. We use the learned models, $\hat{M}$, in value iteration to obtain a policy over options $\pi_{\mathcal{O}}(o|s_t)$. Affordances can be incorporated into planning by only considering state-option pairs in the affordance set (See Algorithm \ref{alg:SMDP-QVI} in the Appendix). We report the \emph{success rate}, i.e., the proportion of episodes in which the agent successfully drops the passenger at the correct location. Data collection, learning, and evaluation happen asynchronously and simultaneously using the Launchpad framework (Fig~\ref{fig:pipeline}, \citet{yang2021launchpad}).

\subsection{When are intents and affordances most useful?}
\label{sec:experiment1}
In this section we investigate the utility of using affordances on different aspects of the pipeline by considering a fixed set of affordances used either during model learning or planning. We first define three intent sets, $\TEISet$, and their corresponding affordances: %
\begin{enumerate*}
    \item \textbf{Everything}: All options are affordable at every state resulting in 37,500 state-option pairs in this affordance set.
    \item \textbf{Pickup+Drop}:  We build this set of affordances heuristically, by eliminating all options that simply go to a grid position, resulting in 25,000  state-option pairs .
    \item \textbf{Pickup+Drop@Goal}: We create this affordance set of 4,000 state-option pairs that terminate at the four destination positions only. 
\end{enumerate*}
When learning the partial model, using the most restrictive affordance set (\textbf{Pickup+Drop@Goal}) to collect data and mask the loss (*$\rightarrow$~Everything) significantly improves the sample efficiency (Fig.~\ref{fig:heuristic_model_learning_only}). The difference between \textbf{Everything} and \textbf{Pickup+Drop} was insignificant suggesting that the order of magnitude decrease in the number of state-option pairs in the affordance set is important (See also Sec~\ref{sec:experiment2} for more analysis of the affordance set size). Additionally, using any affordance set enables the use of a higher learning rate for learning the model without divergence (Fig.~\ref{fig:heuristic_affordances_learning_rate}). On the other hand, given the same option model, using affordance sets only during planning (Everything$\rightarrow$*) does not create any improvement in the success rate (Fig.~\ref{fig:heuristic_planning_only}): the quality of the model dictates the success rate.

Finally, using the most restrictive affordance set for both model learning and planning (Pickup+Drop@Goal$\rightarrow$Pickup+Drop@Goal) can result in further improvements in the sample efficiency (Fig.~\ref{fig:heuristic_both}) demonstrating a combined benefit of using affordances in more aspects of the pipeline.

\begin{figure*}
    \begin{center}
    \subfigure[{Data collection and model learning with affordances.}]{\label{fig:heuristic_model_learning_only}\includegraphics[width=0.32\textwidth]{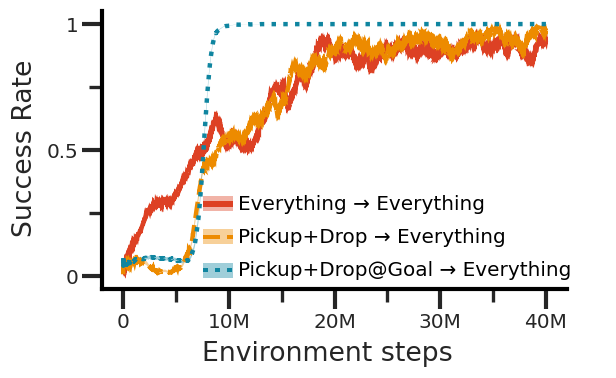}}
    \subfigure[{Planning with affordances.}]{\label{fig:heuristic_planning_only}\includegraphics[width=0.32\textwidth]{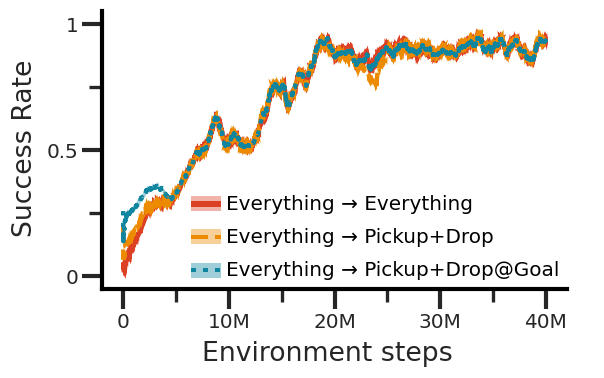}}
    \subfigure[{Data collection, model learning and planning with affordances.}]{\label{fig:heuristic_both}\includegraphics[width=0.32\textwidth]{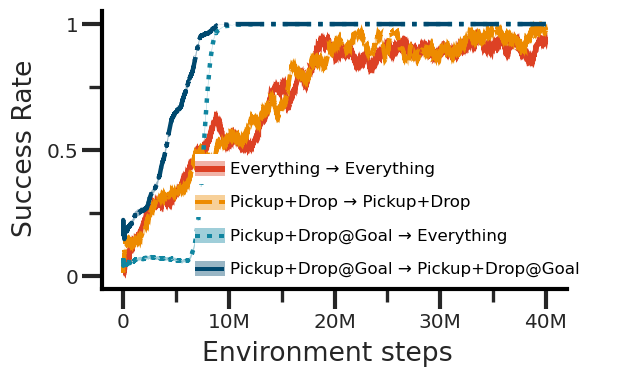}}
    \end{center}
    \vspace{-1\baselineskip}
    \caption{\label{fig:heuristic_affordances}\textbf{The impact of affordance sets on success rate at different parts of the learning pipeline.} (a) The use of affordances improves model learning even in the absence of any affordances during planning (blue). (b) The use of affordances did not impact planning because the underlying quality of the model is the same. (c) When using affordances both during model learning and planning, the best performance is obtained. Curves are smoothed over 4 independent seeds using ggplot's \texttt{stat\_smooth} using a span of 0.1 and confidence interval of 95\%.  %
    }
\end{figure*}

\subsection{Can relevant affordances be learned?}
\label{sec:experiment2}
In this section, we demonstrate the ability to learn affordances at the same time as learning the partial option model.
To do this, we train a classifier, $A_\theta(s, o, s', I)\in[0, 1]$ corresponding to intent $I\in\TEISet$, which predicts if a state-option pair is affordable. \textbf{Pickup+Drop@Goal} is defined by 8 intents: four that are completed when the agent has a passenger in the vehicle at the destinations; and four that are completed when the agent has dropped the passenger at the destinations. We convert $A_\theta(s,o,s',I)$ into an indicator for Eq.~\ref{eqn:masked_model_loss}, by ensuring that at least one of the intents in the intent set is affordable, $A(s,o,s',\TEISet)=\mathbbm{1}[(\max_{I\in\TEISet}(A(s,o,s',I))>k]$ at some threshold value, $k$. When $k =0$, all state and options are affordable. The affordance classifier is learned at the same time as the option model, $\hat{M}$, using the standard cross entropy objective: $-\sum_{I\in\TEISet} c(s,o,s',I) \log A(s,o,s',I)$ where $c(s,o,s',I)$ is the intent completion function indicating if intent $I$ was completed during the transition.
\begin{figure*}[t]
    \begin{center}
    \vspace{-1\baselineskip}
    \subfigure[]{\label{fig:learned_aff_set_size}\includegraphics[width=0.45\textwidth]{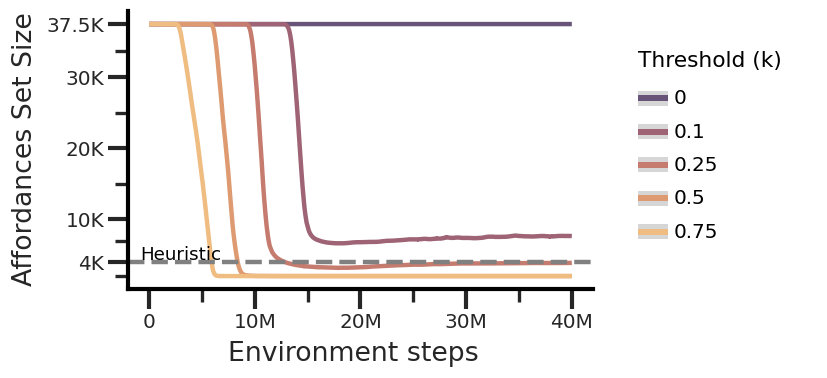}}
    \subfigure[]{\label{fig:learned_prop_succ}\includegraphics[width=0.3\textwidth]{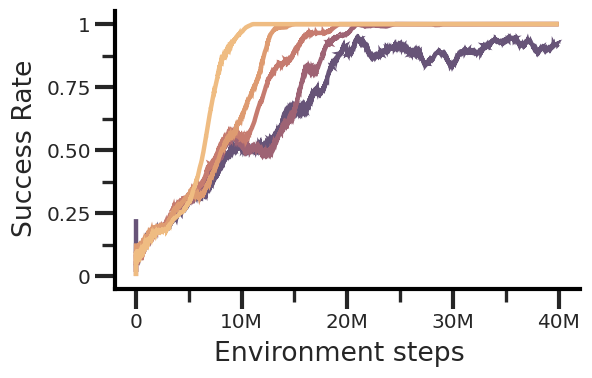}}
    \end{center}
            \vspace{-1\baselineskip}
    \caption{\label{fig:learned_affordances}\textbf{The impact of learning the affordance set for Pickup+Drop@Goal on (a) size of the affordance set and (b) success in the downstream task.} There is a one-to-one correspondence between the threshold, $k$, the affordance set size and the success rate on the taxi task. The learned affordance set for Pickup+Drop@Goal is smaller than the heuristic used in Fig.~\ref{fig:heuristic_both}.}
\end{figure*}

The threshold, $k$, controls the size of the affordance set (Fig.~\ref{fig:learned_aff_set_size}) with larger $k$'s resulting in smaller affordance sets.  The learned affordance set for \textbf{Pickup+Drop@Goal} is 2,000 state-option pairs which smaller than what we heuristically defined (4,000 state-option pairs). Smaller affordance sets result in improved sample efficiency (Fig.~\ref{fig:learned_prop_succ}).
We highlight that this is not necessarily obvious since the learned affordance sets could remove potentially useful state-options pairs and $k$ would be used to control how restrictive the sets are. These results show that affordances can be learned online for a defined set of intents and result in good performance. In particular, there are sample efficiency gains by using more restricted affordance sets.

Our results here demonstrate empirically that learning a partial option model requires much fewer samples as opposed to learning a full model. We also corroborate this with theoretical guarantees on sample and computational complexity of obtaining an $\varepsilon$-estimation of the optimal option value function, given only access to a generative model (See Appendix Sec.~\ref{sec:samplecomplexity}).

\section{Related Work}\label{sec:related}
Affordances are viewed as the action opportunities~\citep{gibson1977theory, chemero2003outline}, emerging out of the agent-environment interaction~\citep{heft1989affordances}, and have been typically studied in AI as possibilities associated with an object~\citep{slocum2000further, fitzpatrick2003learning, lopes2007affordance, montesano2008learning, cruz2016training, cruz2018multi, fulda2017can, song2015learning, abel2014toward}. Affordances have also been formalized in RL without the assumption of objects~\citep{khetarpal2020i}. Our work presents the general case of temporal abstraction~\citep{sutton1999between}.

The process model of behavior and cognition~\citep{pezzulo2016navigating} in the space of affordances is expressed at multiple levels of abstraction. During interactive behavior, action representations at different levels of abstraction can indeed be mapped to findings about the way in which the human brain adaptively selects among predictions of outcomes at different time scales~\citep{cisek2010neural, pezzulo2016navigating}.

In RL, the generalization of one-step action models to option models~\citep{sutton1999between} enables an agent to predict and reason at multiple time scales. \citet{precup1998theoretical} established dynamic programming results for option models which enjoy similar theoretical guarantees as primitive action models. \citet{abel2019expected} proposed expected-length models of options. Our theoretical results can also be extended to expected-length option  models. %

Building agents that can represent and use predictive knowledge requires efficient solutions to cope with the combinatorial explosion of possibilities, especially in large environments. Partial models~\citep{talvitie2009simple} provide an elegant solution to this problem, as they only model part of the observation. Some existing methods focus on predictions for only some of the observations~\citep{oh2017value,amos2018learning,guo2018neural,gregor2019shaping}, but they still model the effects of all the actions and focus on  single-step dynamics~\citep{watters2019cobra}. Recent work by \citet{xu2020deep} proposed a deep RL approach to learn partial models with goals akin to intents, which is complementary to our work.

\section{Conclusions and Limitations}\label{sec:concl}

We presented notions of intents and affordances that can be used together with options. They allow us to define {\em temporally abstract partial models}, which extend option models to be conditioned on affordances. Our theoretical analysis suggests that modelling temporally extended dynamics for only relevant parts of the environment-agent interface provides two-fold benefits: 1) faster planning across different timescales (Sec.~\ref{sec:valueandplanningloss}), and 2) improved sampled efficiency (Appendix Sec.~\ref{sec:samplecomplexity}).However, these benefits can come at the cost of some increase in approximation bias, but this tradeoff can still be favourable. For example, in the low-data regime, intermediate-size affordances (much smaller than the entire state-option space) could really improve the speed of planning. Picking intents judiciously can also induce sample complexity gains, if the approximation error due to the intent is manageable. Our empirical illustration shows that our approach can produce significant benefits.

\textbf{Limitations \& Future Work.} Our analysis assumes that the intents and options are fixed apriori. To learn intents, we envisage an iterative algorithm which alternates between learning intents and affordances, such that intents can be refined over time and the mis-specifications can also be self-corrected~\citep{talvitie2017self}. Our analysis is complimentary to any method for providing or discovering intents. 
Another important future direction is to build partial option models and leverage their predictions in large scale problems, such as~\citep{vinyals2019grandmaster}.
Besides, it would be useful to relate our work to cognitive science models of {\em intentional options}, which can reason about the space of future affordances~\citep{pezzulo2016navigating}. Aligned with future affordances, a promising research avenue is to study the emergence of \emph{new} affordances at the boundary of the agent-environment interaction in the presence of non-stationarity~\citep{chandak2020lifelong}.

\begin{ack}
The authors would like to thank Feryal Behbahani and Dave Abel for a very detailed feedback, Martin Klissarov and Emmanuel Bengio for valuable comments on a draft of this paper, and Joelle Pineau for feedback on ideas presented in this work. A special thank you to Ahmed Touati for discussion and detailed notes~\citep{azar2012sample} presented in RL theory reading group at Mila. 
\end{ack}

\bibliography{references}
\bibliographystyle{apalike}

\appendix

\onecolumn
\renewcommand\thefigure{\thesection\arabic{figure}}  
\setcounter{section}{0}
\setcounter{figure}{0}
\section*{\centering Appendix \\
\underline {Temporally Abstract Partial Models}}

\section{Proofs}
\label{sec:appendix}

\subsection{Lemmas and Remarks}
\label{sec-proofs}

\subsubsection{Proof of Lemma~\ref{ass:generalzetaconstant}}
\label{proof-lemma1}

{\bf Lemma 1: }
Given  a set of temporally extended intents ${\TEISet}=\cup_{o \in \Omega} \TEI$ that are satisfied to degrees $\zeta_{s,o}$, there exist constants ($\zeta^{\mathcal{I}^{\rightarrow}}_P, \zeta^{\mathcal{I}^{\rightarrow}}_R$) such that:
\begin{align*}
     \textstyle \max_{s, o} \sum_{\tau(s)\in {\cal T}(s)} \Big| P_o(\tau(s)) -  P_I(\tau(s))) \Big| & \leq \zeta^{\mathcal{I}^{\rightarrow}}_P \mbox{ and } \\
    \textstyle  \max_{s, o}  \Big| r(s,o) - E_{\tau \sim P_I} [G(\tau | s,o )] \Big|  & \leq \zeta^{\mathcal{I}^{\rightarrow}}_R
\end{align*}

\begin{proof} (Approximate Probability Distributions)
From Def.~\ref{definitionintentgeneral}, $\forall \TEI \in \TEISet$, $\TEI$ is satisfied to a degree, $\zeta_{s,o}$ at state $s\in \cal \mathcal{S}$ and option $o \in \mathcal{O}$ if and only if:
\begin{equation*}
   d(P_I(\tau| s, o), P_o(\tau(s)))\leq \zeta_{s,o},
\end{equation*}
where $d$ is a metric between probability distributions. Let  $\zeta^{\mathcal{I}^{\rightarrow}}_P = \max_{s,o} \zeta_{s,o}$. The result follows immediately.
\end{proof}

\begin{proof} (Approximate Reward Distributions)
Let $\zeta^{\mathcal{I}^{\rightarrow}}_R = \big| \big| G \big| \big|_\infty \zeta^{\mathcal{I}^{\rightarrow}}_P $.
We now consider the maximum error in approximation of rewards due to intent specification as follows:
\begin{align*}
\begin{split}
    &\textstyle  \max_{s, o}  \Big| r(s,o) - E_{\tau \sim P_I} [G(\tau | s,o )] \Big| \\
    &= \max_{s, o}  \Big| \sum_{s'} r(s,o,s') - \sum_{s'} \sum_{\tau} \sum_{t=1}^{\infty} P_I(\tau(s,t,s')|s, o)) G (\tau(s,t,s'))\Big| \\
    &= \max_{s, o}  \Big| \sum_{s'} \sum_{\tau} \sum_{t=1}^{\infty} P_o(\tau(s,t,s')|s, o)) G (\tau(s,t,s')) - \\
    &\sum_{s'} \sum_{\tau} \sum_{t=1}^{\infty} P_I(\tau(s,t,s')|s, o)) G (\tau(s,t,s'))\Big| \\
    &= \max_{s, o}  \Big| \sum_{s'}  \sum_{t=1}^{\infty} \Big( \sum_{\tau} P_o(\tau(s,t,s')|s, o)) - P_I(\tau(s,t,s')|s, o)) \Big) G (\tau(s,t,s'))\Big|\\
    &\leq \Big| \Big| G \Big| \Big|_\infty \zeta^{\mathcal{I}^{\rightarrow}}_P = \zeta^{\mathcal{I}^{\rightarrow}}_R
\end{split}
\end{align*}

\end{proof}

\subsubsection{Remarks}
\label{sec:remark1proof}
\begin{remark}
\label{remark:generalupperboundonvstar}
Given a finite SMDP $\mathcal{M}$, a finite set of options $\mathcal{O}$, the maximum achievable optimal value function $\Big| \Big|V^*\Big| \Big|_\infty$ is upper bounded by $ \frac{R^{\mathcal{O}}_{max}}{ \left(1- \upgamma^{\cal O} \right)}$ where $
\upgamma^{\cal O} = \max_{s,o} \sum_{s'} \gamma_o(s,s')$, and $R^{\mathcal{O}}_{max} = \max_{s,o} R(s,o)$.
\end{remark}

\begin{proof}
To upper bound the optimal value function, we consider $\Big| \Big|Q^*\Big| \Big|_\infty = \max_{s,o} Q^{*}(s,o) = \max_{s} \underbrace{\max_{o} Q^{*}(s,o)}_{V^*}$. Then, $\forall$ $s , o \in \mathcal{S}, \mathcal{O}$ :
\begin{align*}
    Q^{*}(s,o) &=  \sum_{s'} \sum_{t=1}^\infty \sum_{\tau(s,t,s')}  P(\tau(s,t,s')|o) [  G(\tau(s,t,s') + \gamma(\tau(s,t,s')) \max_{o'} Q^*(s', o') ] \\
    &= R(s,o) +  \sum_{s'} \sum_{t=1}^\infty \sum_{\tau(s,t,s')} P(\tau(s,t,s')|o) \gamma(\tau(s,t,s')) \max_{o'} Q^*(s', o')
\end{align*}

Taking the max norm on both sides,
\begin{align*}
    \max_{s,o} Q^{*}(s,o) &= \Big| \Big| R(s,o) +  \sum_{s'} \sum_{t=1}^\infty \sum_{\tau(s,t,s')} P(\tau(s,t,s')|o) \gamma(\tau(s,t,s')) \max_{o'} Q^*(s', o') \Big| \Big|_\infty \\
    &\leq \max_{s, o} R(s,o) + \max_{s,o} \underbrace{ \sum_{s'} \sum_{t=1}^\infty \sum_{\tau(s,t,s')} P(\tau(s,t,s')|o) \gamma(\tau(s,t,s'))}_{ \sum_{s'}\gamma_o(s,s')} \max_{o'} Q^*(s', o') \\
    &\leq R^{\mathcal{O}}_{max} + \Big| \Big|Q^*\Big| \Big|_\infty \max_{s,o} \sum_{s'}\gamma_o(s,s')   \\
    &\implies \Big| \Big|Q^*\Big| \Big|_\infty \leq R^{\mathcal{O}}_{max} \left( 1- \max_{s,o} \sum_{s'} \gamma_o(s,s') \right)^{-1} \\
    &\implies \Big| \Big|V^*\Big| \Big|_\infty \leq R^{\mathcal{O}}_{max} \left( 1- \max_{s,o} \sum_{s'} \gamma_o(s,s') \right)^{-1} = R^{\mathcal{O}}_{max} \left( 1- \upgamma^{\cal O} \right)^{-1}.
\end{align*}
\end{proof}

\label{sec:remark2proof}
\begin{remark}
\label{remark:smdpvmaxupperbound}
Given a finite SMDP $\mathcal{M}$, a finite set of options $\mathcal{O}$, with  $\mathcal{D}$ as the minimum expected duration for which all options execute, $\upgamma$ to be the maximum expected option discount factor, the maximum achievable optimal value function $V_{max}$ is upper bounded by $\frac{R^{\mathcal{O}}_{max}}{(1-\gamma^{\mathcal{D}})} = \frac{R^{\mathcal{O}}_{max}}{(1-\upgamma^{\cal O})}$ , where $R^{\mathcal{O}}_{max}$ is the maximum achievable reward by an option, and $\mathcal{D} = \min_{s,o} \log_\gamma \sum_{s'} p(s' |s, o) $.
\end{remark}
\begin{proof}
Consider the maximum achievable optimal value function in the SMDP $\mathcal{M}$ to be $ V_{max}$.
\begin{align*}
    V_{max} = ||V^{*}_\mathcal{O}||_\infty.
\end{align*}

Then, $\forall$ $s \in \mathcal{S}$:
\begin{align*}
    V^{*}_\mathcal{O}(s) &= \max_{o \in \mathcal{O}} \Big[  R(s,o) + \sum_{s'} p(s'|s,o) V^{*}_\mathcal{O}(s') \Big] \\
    &\leq \max_{o \in \mathcal{O}} \Big[  R(s,o) + \sum_{s'} p(s'|s,o) \max_{s'' \in \mathcal{S}} V^{*}_\mathcal{O}(s'') \Big] \\
    & = \max_{o \in \mathcal{O}} \Big[  R(s,o) + \gamma_o(s) \max_{s'' \in \mathcal{S}} V^{*}_\mathcal{O}(s'') \Big],  \text{\;\; substituting } \gamma_o(s) = \sum_{s'} p(s'|s,o) \\
    & = \max_{o \in \mathcal{O}} \Big[  R(s,o) + \gamma_o(s) || V^{*}_\mathcal{O} ||_\infty \Big] \\
    &\leq \underbrace{\max_{o \in \mathcal{O}}  R(s,o)}_{ \leq R^{\mathcal{O}}_{max}} + \underbrace{\max_{s,o \in \mathcal{S},\mathcal{O}} \gamma_o(s)}_{\leq \gamma_{max}} || V^{*}_\mathcal{O} ||_\infty \\
    &\leq  R^{\mathcal{O}}_{max} + \gamma_{max} || V^{*}_\mathcal{O} ||_\infty, \text{\;\; where } R^{\mathcal{O}}_{max} = \max_{s,o} R(s,o) \\
    &\implies ||V^{*}_\mathcal{O}||_\infty \leq \frac{R^{\mathcal{O}}_{max}}{(1-\gamma_{max})}.
\end{align*}

Note consider the following definition of $\mathcal{D}$:
\begin{align*}
    \mathcal{D}  &= \min_{s,o} \log_\gamma \sum_{s'} p(s' |s, o) = \min_{s,o} \log_\gamma \gamma_o(s) \\
    &=  \log_\gamma \underbrace{\max_{s,o} \gamma_o(s)}_{\gamma_{max}}, \text{\;\; since } \gamma < 1, \log_\gamma \text{is a monotonically decreasing function} \\
    &= \log_\gamma \gamma_{max} \\
    &\implies \gamma_{max} = \gamma^{\mathcal{D}} \implies \gamma^{\mathcal{D}} = \upgamma^{\cal O}
\end{align*} Therefore, $V_{max} \leq \frac{R^{\mathcal{O}}_{max}}{(1-\upgamma^{\cal O})}.$
\end{proof}

\subsection{Proofs - Value Loss Analysis}
\label{appsec:valueloss}

\textbf{Note:} For convenience, throughout our proofs we will be using ${\mathcal{I}}$ instead of $\TEISet$ to denote a set of temporally extended intents. Similarly, we will use $I$ instead of $\TEI$ to denote a temporally extended intent for an option $o$.

\subsubsection{\textbf{Proof of Theorem~\ref{thm:trajectories_value_loss_analysis}}}
\label{sec:app-trajectories_value_loss_analysis}

\begin{proof}
Formally, the value loss is defined as \begin{equation*}
    \Big| \Big| V^{\pi^{*}_{\cal I}}_{\cal M} - V^*_{\cal M} \Big| \Big|_{\infty}= \max_{s \in \cal S} \Big| V^{\pi^{*}_{\cal I}}_{\cal M}(s) - V^*_{\cal M}(s) \Big|
\end{equation*}
We now consider the RHS and expand as follows: \begin{align*}
    \max_{s \in \cal S}  \Big| V^{\pi^{*}_{\cal I}}_{\cal M}(s) - V^*_{\cal M}(s) \Big| \leq \underbrace{\max_{s \in \cal S} \Big| V^*_{\cal M}(s) - V^{*}_{\cal M_{\cal I}}(s) \Big|}_\textbf{Term 1} + \underbrace{\max_{s \in \cal S} \Big| V^{\pi^{*}_{\cal I}}_{\cal M}(s) - V^*_{\cal M_{\cal I}}(s) \Big|}_\textbf{Term 2}
\end{align*}

\textit{Bounding Term 1.} \begin{align*}
    \textstyle \max_{s \in \cal S} \Big| V^*_{\cal M}(s) - V^{*}_{\cal M_{\cal I}}(s) \Big| &= \max_{s \in \cal S} \max_{o \in \cal O} \Big| Q^*(s,o) - Q^*_I(s,o) \Big|
\end{align*}

Expanding the action-value loss from the RHS above, we get:
\begin{align*}
& Q^*(s,o) - Q^*_I(s,o) = \\
& \quad \quad = \sum_{s'} \sum_{t=1}^\infty \sum_{\tau(s,t,s')} P(\tau(s,t,s')|o) [  G(\tau(s,t,s') + \gamma(\tau(s,t,s')) \max_{o'} Q^*(s', o') ] \\ & \quad \quad \quad \quad - \sum_{s'} \sum_{t=1}^\infty \sum_{\tau(s,t,s')} P_I(\tau(s,t,s')|o)[ G(\tau(s,t,s') + \gamma(\tau(s,t,s')) \max_{o'} Q_I^*(s', o')]\\
& \quad \quad = \sum_{s'} \sum_{t=1}^\infty \sum_{\tau(s,t,s')} (P(\tau(s,t,s')|o) - P_I(\tau(s,t,s')|o) G(\tau(s,t,s')\\ 
& \quad \quad \quad \quad + \sum_{s'} \sum_{t=1}^\infty \sum_{\tau(s,t,s')} \gamma(\tau(s,t,s')) \Big( P(\tau(s,t,s')|o)  \max_{o'} Q^*(s', o') - P_I(\tau(s,t,s')|o) \max_{o'} Q_I^*(s', o') \Big)
\\
& \quad \quad = (R(s,o)-R_I(s,o)) + \sum_{s'} \sum_{t=1}^\infty \sum_{\tau(s,t,s')} \gamma(\tau(s,t,s')) \Big( P(\tau(s,t,s')|o)  \max_{o'} Q^*(s', o') \\
& \quad \quad \quad \quad - P_I(\tau(s,t,s')|o) \max_{o'} Q^*(s', o') + P_I(\tau(s,t,s')|o) \max_{o'} Q^*(s', o') - P_I(\tau(s,t,s')|o) \max_{o'} Q_I^*(s', o') \Big)\\
& \quad \quad = (R(s,o)-R_I(s,o)) + \sum_{s'} \sum_{t=1}^\infty \sum_{\tau(s,t,s')} \gamma(\tau(s,t,s')) (P(\tau(s,t,s')|o) - P_I(\tau(s,t,s')|o)) \max_{o'} Q^*(s', o') \\
& \quad \quad \quad \quad + \sum_{s'} \sum_{t=1}^\infty \sum_{\tau(s,t,s')} \gamma(\tau(s,t,s')) P_I(\tau(s,t,s')|o))  \max_{o'} ( Q^*(s', o') - Q_I^*(s', o') )
\end{align*} Taking the max norm and applying triangle inequality, we get:\begin{align*}
\begin{split}
\Big| \Big| Q^* - Q^*_I \Big| \Big|_\infty &= \max_{s,o} \Big[ (R(s,o)-R_I(s,o)) + \\
&\sum_{s'} \sum_{t=1}^\infty \sum_{\tau(s,t,s')} \gamma(\tau(s,t,s')) (P(\tau(s,t,s')|o) - P_I(\tau(s,t,s')|o)) \max_{o'} Q^*(s', o') \\
&+ \sum_{s'} \sum_{t=1}^\infty \sum_{\tau(s,t,s')} \gamma(\tau(s,t,s')) P_I(\tau(s,t,s')|o))  \max_{o'} ( Q^*(s', o') - Q_I^*(s', o') ) \Big] \\ 
&\leq \big| \big| R - R_I \big| \big|_\infty + \max_{s,o} \sum_{s'} \sum_{t=1}^\infty \sum_{\tau(s,t,s')} \gamma(\tau(s,t,s')) \Big( P(\tau(s,t,s')|o) - P_I(\tau(s,t,s')|o) \Big) ||Q^*||_\infty \\ 
&+ \max_{s,o} \underbrace{\sum_{s'} \sum_{t=1}^\infty \sum_{\tau(s,t,s')} \gamma(\tau(s,t,s'))  P_I(\tau(s,t,s')|o))}_{\sum_{s'} \gamma_o^I(s,s')}  \Big| \Big| Q^* - Q_I^* \Big| \Big|_\infty \\
&\leq \big| \big| R - R_I \big| \big|_\infty + \max_{s,o} \sum_{s'} \sum_{t=1}^\infty \sum_{\tau(s,t,s')} \gamma(\tau(s,t,s')) \Big(  P(\tau(s,t,s')|o) - P_I(\tau(s,t,s')|o) \Big) \Big| \Big|Q^*\Big| \Big|_\infty \\ 
&+ \max_{s,o} \sum_{s'} \gamma_o^I(s,s')  \Big| \Big| Q^* - Q_I^* \Big| \Big|_\infty
\end{split}
\end{align*}

Rearranging, we get: \begin{align*}
    \Big| \Big| Q^* - Q^*_I \Big| \Big|_\infty &\leq  \Big( 1-\max_{s,o} \sum_{s'} \gamma_o^I(s,s') \Big)^{-1} \Big[ \big|  \big| R-R_I \big| \big|_\infty + \\  & \quad \quad \max_{s,o} \sum_{s'} \sum_{t=1}^\infty \sum_{\tau(s,t,s')} \gamma(\tau(s,t,s')) \big( P(\tau(s,t,s')|o) - P_I(\tau(s,t,s')|o)) \big) \big| \big|Q^*\big| \big|_\infty  \Big]
\end{align*}

Since $V^*(s) = \max_o Q^*(s,o)$, we can rewrite the above as following:\begin{align*}
    \Big| \Big| V^*_{\cal M} - V^{*}_{\cal M_{\cal I}} \Big| \Big|_\infty &\leq  \Big( 1-\max_{s,o} \sum_{s'} \gamma_o^I(s,s') \Big)^{-1} \Big[ \big|  \big| R-R_I \big| \big|_\infty + \\  & \quad \quad + \max_{s,o} \sum_{s'} \sum_{t=1}^\infty \sum_{\tau(s,t,s')} \gamma(\tau(s,t,s')) \big( P(\tau(s,t,s')|o) - P_I(\tau(s,t,s')|o)) \big) \big| \big|V^*\big| \big|_\infty  \Big]
\end{align*}

\vspace{-.5cm}
\textit{Bounding Term 2.}
We now consider the term 2 and bound the policy evaluation error i.e. $\max_{s \in \cal S} \Big| V^{\pi^{*}_{\cal I}}_{\cal M}(s) - V^{\pi^{*}_{\cal I}}_{\cal M_{\cal I}}(s) \Big|$ 
\begin{align*}
& V^{\pi^{*}_{\cal I}}_{\cal M}(s) - V^{\pi^{*}_{\cal I}}_{\cal M_{\cal I}}(s) = \\
& \quad \quad = \sum_{s'} \sum_{t=1}^\infty \sum_{\tau(s,t,s')} P(\tau(s,t,s')|\pi^{*}_{\cal I}(s)) [  G(\tau(s,t,s')) + \gamma(\tau(s,t,s')) V^{\pi^{*}_{\cal I}}_{\cal M}(s') ] \\ 
& \quad \quad \quad \quad - \sum_{s'} \sum_{t=1}^\infty \sum_{\tau(s,t,s')} P_I(\tau(s,t,s')|\pi^{*}_{\cal I}(s))[ G(\tau(s,t,s')) + \gamma(\tau(s,t,s')) V^{\pi^{*}_{\cal I}}_{\cal M_{\cal I}}(s')]\\
& \quad \quad = \sum_{s'} \sum_{t=1}^\infty \sum_{\tau(s,t,s')} \Big( P(\tau(s,t,s')|\pi^{*}_{\cal I}(s)) - P_I(\tau(s,t,s')|\pi^{*}_{\cal I}(s) \Big) G(\tau(s,t,s'))\\ 
& \quad \quad \quad \quad + \sum_{s'} \sum_{t=1}^\infty \sum_{\tau(s,t,s')} \gamma(\tau(s,t,s')) \Big( P(\tau(s,t,s')|\pi^{*}_{\cal I}(s)) V^{\pi^{*}_{\cal I}}_{\cal M}(s') - P_I(\tau(s,t,s')|\pi^{*}_{\cal I}(s)) V^{\pi^{*}_{\cal I}}_{\cal M_{\cal I}}(s') \Big)
\\
& \quad \quad = (R(s,\pi^{*}_{\cal I}(s))-R_I(s,\pi^{*}_{\cal I}(s))) + \sum_{s'} \sum_{t=1}^\infty \sum_{\tau(s,t,s')} \gamma(\tau(s,t,s')) \Big( P(\tau(s,t,s')|\pi^{*}_{\cal I}(s)) V^{\pi^{*}_{\cal I}}_{\cal M}(s') \\
& \quad \quad \quad \quad - P_I(\tau(s,t,s')|\pi^{*}_{\cal I}(s)) V^{\pi^{*}_{\cal I}}_{\cal M}(s') + P_I(\tau(s,t,s')|\pi^{*}_{\cal I}(s)) V^{\pi^{*}_{\cal I}}_{\cal M}(s') - P_I(\tau(s,t,s')|\pi^{*}_{\cal I}(s)) V^{\pi^{*}_{\cal I}}_{\cal M_{\cal I}}(s') \Big)\\
& \quad \quad = (R(s,\pi^{*}_{\cal I}(s))-R_I(s,\pi^{*}_{\cal I}(s))) + \sum_{s'} \sum_{t=1}^\infty \sum_{\tau(s,t,s')} \gamma(\tau(s,t,s')) (P(\tau(s,t,s')|\pi^{*}_{\cal I}(s)) - P_I(\tau(s,t,s')|\pi^{*}_{\cal I}(s))) V^{\pi^{*}_{\cal I}}_{\cal M}(s') \\
& \quad \quad \quad \quad + \sum_{s'} \sum_{t=1}^\infty \sum_{\tau(s,t,s')} \gamma(\tau(s,t,s')) P_I(\tau(s,t,s')|\pi^{*}_{\cal I}(s))) \Big( V^{\pi^{*}_{\cal I}}_{M}(s') - V^{\pi^{*}_{\cal I}}_{\cal M_{\cal I}}(s') \Big)
\end{align*}

Taking the max over all states, and applying triangle inequality we get:
\begin{align*}
&\max_{s \in \cal S} \Big| V^{\pi^{*}_{\cal I}}_{\cal M}(s) - V^{\pi^{*}_{\cal I}}_{\cal M_{\cal I}}(s) \Big| = \\
& \max_{s} \Big| (R(s,\pi^{*}_{\cal I}(s))-R_I(s,\pi^{*}_{\cal I}(s))) \\
& + \sum_{s'} \sum_{t=1}^\infty \sum_{\tau(s,t,s')} \gamma(\tau(s,t,s')) (P(\tau(s,t,s')|\pi^{*}_{\cal I}(s)) - P_I(\tau(s,t,s')|\pi^{*}_{\cal I}(s))) V^{\pi^{*}_{\cal I}}_{\cal M}(s') \\
& + \sum_{s'} \sum_{t=1}^\infty \sum_{\tau(s,t,s')} \gamma(\tau(s,t,s')) P_I(\tau(s,t,s')|\pi^{*}_{\cal I}(s))) \Big( V^{\pi^{*}_{\cal I}}_{M}(s') - V^{\pi^{*}_{\cal I}}_{\cal M_{\cal I}}(s') \Big) \Big| \\ 
&\leq \big| \big| R - R_I \big| \big|_\infty \\
&+ \max_s  \sum_{s'} \sum_{t=1}^\infty \sum_{\tau(s,t,s')} \gamma(\tau(s,t,s')) |P(\tau(s,t,s')|\pi^{*}_{\cal I}(s)) - P_I(\tau(s,t,s')|\pi^{*}_{\cal I}(s))) \Big| \Big|V^{\pi^{*}_{\cal I}}_{\cal M}\Big| \Big|_\infty \\ 
&+ \max_s \sum_{s'} \sum_{t=1}^\infty \sum_{\tau(s,t,s')} \gamma(\tau(s,t,s')) P_I(\tau(s,t,s')|\pi^{*}_{\cal I}(s)))  \Big| \Big| V^{\pi^{*}_{\cal I}}_{M}(s') - V^{\pi^{*}_{\cal I}}_{\cal M_{\cal I}} \Big| \Big|_\infty
\end{align*}

Rearranging the terms, we get: 
\begin{align*}
    \Big| \Big| V^{\pi^{*}_{\cal I}}_{\cal M} - V^{\pi^{*}_{\cal I}}_{\cal M_{\cal I}} \Big| \Big|_\infty &\leq \Big( 1-\max_{s,o} \sum_{s'} \gamma_o^I(s,s') \Big)^{-1} \Big[ \big|  \big| R-R_I \big| \big|_\infty + \\  & \quad \quad + \max_{s,o} \sum_{s'} \sum_{t=1}^\infty \sum_{\tau(s,t,s')} \gamma(\tau(s,t,s')) \big| P(\tau(s,t,s')|o) - P_I(\tau(s,t,s')|o)| \big) \big| \big|V^*\big| \big|_\infty  \Big]
\end{align*}

Plugging the bounds for the two terms in our original loss, and plugging the upper bound on the optimal value function from Remark~\ref{remark:generalupperboundonvstar}, we get:
\begin{align*}
    \Big| \Big| V^{\pi^{*}_{\cal I}}_{\cal M} - V^*_{\cal M} \Big| \Big|_{\infty} &\leq \Big( 1-\max_{s,o} \sum_{s'} \gamma_o^I(s,s') \Big)^{-1} \Big|  \Big| R-R_I \Big| \Big|_\infty + \frac{2 R^{\mathcal{O}}_{max} \Big( 1-\max_{s,o} \sum_{s'} \gamma_o^I(s,s') \Big)^{-1}}{\Big( 1-\max_{s,o} \sum_{s'} \gamma_o(s,s') \Big)} \times \\
    &\max_{s,o} \sum_{s'} \sum_{t=1}^\infty \sum_{\tau(s,t,s')} \gamma(\tau(s,t,s')) \Big| P(\tau(s,t,s')|o) - P_I(\tau(s,t,s')|o)) \Big|
\end{align*}

Further, substituting Lemma~\ref{ass:generalzetaconstant}, we get the final result as follows:
\begin{align*}
    \Big| \Big| V^{\pi^{*}_{\cal I}}_{\cal{M}} - V^*_{\cal{M}} \Big| \Big|_{\infty}  &\leq \frac{\zeta^{\mathcal{I}}_R}{\Big( 1-\upgamma^{\cal I} \Big)} + \frac{2 R^{\mathcal{O}}_{max} \max_{s,o} \sum_{t=1}^\infty \gamma^t  |\mathcal{S}| \zeta^{\mathcal{I}}_P %
    }{\Big( 1-\upgamma^{\cal I} \Big)\Big( 1-\upgamma \Big)} 
\end{align*}
Recall that ${\cal I}$ was used to denote $\TEISet$, the set of temporally extended intents, throughout the proof.
\end{proof}

\subsubsection{Corollary 1. SMDP - Multi-Time-Model of Intent - Value Loss Bound}
\label{sec:app-smdpvalueloss}

A special case of our formulation is to model the consequences of following a specific course of action based on final state representations at the SMDP level. 

More precisely, the multi-time-model of an option intent must characterize both the target state distribution resulting upon the option's completion, and the intended temporal scale at which the option operates i.e. $\TEI: \cal{S} \to \text{SDist}(\cal{S})$, where $\text{SDist}$ stands for the set of all sub-probability distributions over ${\cal S}$. The intent-induced transition model would then take the role of the transition dynamics reflected by the option model (assuming rewards are the same and known). For this case, we require a metric between sub-probability distributions and assume that,

\begin{assumption}
\label{ass:zetaconstant}
For each state-option pair, the total variation between the intended distribution $P_I$ and the true distribution $P$ is bounded by a constant $\zeta_{s,o}$, i.e.
\begin{align}
\label{eq:mtmzetaconstant}
    \textstyle \sum_{s'} \Big| P_{I}(s' | s, o) - p(s'| s, o) \Big| \leq \zeta_{{s,o}}.
\end{align}
The degree of satisfaction of the intent is the maximum over all $(s,o)$ pairs, i.e. $\max_{s,o}  \zeta_{{s,o}} = \zeta^{\mathcal{I}}.$
\end{assumption}

\begin{corollary}
\label{corollary-smdp-valueloss} [Multi-Time-Model of Intent- Value Loss.]
Given a SMDP $\mathcal{M}$ corresponding to a set of options $\mathcal{O}$ and a set of temporally extended multi-time-model of intents, the value loss between the optimal policy for the original SMDP $\cal{M}$ and the optimal policy $\pi^*_{\TEISet}$ for the induced SMDP $\cal{M}_{\TEISet}$ is given by:
\begin{equation}
\label{eq:corollarysmdpvalueloss}
    \Big| \Big| V^{\pi^{*}_{\TEISet}}_{\cal{M}} - V^*_{\cal{M}} \Big| \Big|_{\infty} \leq   2\zeta^{\TEISet} \frac{ \upgamma R^{\mathcal{O}}_{max}}{(1-\upgamma)^2}, 
\end{equation}
where $\zeta^{\TEISet}$ is the degree of satisfaction of the intents (Eq.~\ref{eq:mtmzetaconstant}), $R^{\mathcal{O}}_{max}=\max_{s,o} r(s,o)$ is the maximum option reward, and $\upgamma$ is the maximum expected option discount factor.
\end{corollary}

\begin{proof}
We now show that our general result in Theorem~\ref{thm:trajectories_value_loss_analysis} can be reduced to a specific case of considering the multi-time-option model of intents.

We first assume here that rewards are known and given which results in the term $ \big| \big| R-R_I \big| \big|_\infty = 0 $, and the second term can be simplified further as follows:
\begin{align*}
\Big| \Big| Q^* - Q^*_I \Big| \Big|_\infty &\leq  \frac{\big| \big|Q^*\big| \big|_\infty }{1-\max_{s,o} \sum_{s'} \gamma_o^I(s,s')} \max_{s,o} \sum_{s'} \sum_{t=1}^\infty \sum_{\tau(s,t,s')} \gamma(\tau(s,t,s')) \big|  P(\tau(s,t,s')|o) - P_I(\tau(s,t,s')|o) \big|
\end{align*}

Plugging Remark~\ref{remark:generalupperboundonvstar}, we get:
\begin{align*}
    || V^{\pi^{*}_{\mathcal{I}}}_{\cal{M}} - V^*_{\cal{M}}||_\infty \leq \frac{2 R^{\mathcal{O}}_{max}}{(1- \upgamma)^2} \underbrace{ \max_{s,o} \sum_{s'} \sum_{t=1}^\infty \sum_{\tau(s,t,s')} \gamma(\tau(s,t,s')) \Big|  P(\tau(s,t,s')|o) - P_I(\tau(s,t,s')|o) \Big|}_{ \leq \upgamma \zeta^{\mathcal{I}^{\rightarrow}}}
\end{align*}

Simplifying terms, we get the final result
\begin{align*}
      \Big| \Big| V^{\pi^{*}_{\mathcal{I}}}_{\cal{M}} - V^*_{\cal{M}} \Big| \Big|_{\infty} \leq   2\zeta^{\mathcal{I}} \frac{ \upgamma R^{\mathcal{O}}_{max}}{(1-\upgamma)^2}
\end{align*}
\end{proof}

\subsection{Proofs - Planning Loss Analysis}
\label{appsec:planloss}

\ddef{Policy class $\Pi_{\TEISet}$}{Given affordance set $ \AF$, let  $\mathcal{M}_{\TEISet}$ be the set of SMDPs over the state-options pairs in $\AF$, let 
\begin{equation*}
\Pi_{\TEISet} = \{ \pi^{*}_M \} \cup \{ \pi : \exists \bar M \in  \mathcal{M}_{\TEISet} \text{ s.t. }  \pi \text{ is optimal in } \bar M  \}.
\end{equation*}}

\subsubsection{\textbf{Proof of Theorem~\ref{theorem:trajectories_planningvaluelossbound}}. Planning Loss - Trajectories Based Intent.}
\label{sec:app-trajectories_planninglossboundproof}

\begin{proof}
To prove this theorem we will be using the lemmas below: Lemma~\ref{lemma2smdp-trajectories}, Lemma~\ref{lemma3smdp-trajectories}, and Lemma~\ref{lemma4smdp-trajectories}, and \ref{lemma:jiangv2smdp-trajectories}. 

\textbf{Note:} For convenience, throughout our proofs we will be using ${\mathcal{I}}$ instead of $\TEISet$ to denote a set of temporally extended intents. Similarly, we will use $I$ instead of $\TEI$ to denote a temporally extended intent for an option $o$.

\lemma{For any SMDP $\hat{\cal M}_\AFnoarrow$,  which is an approximate model of the SMDP given by the intent collection $\mathcal{I}$\footnote{We overload notation and throughout our proofs, for convenience we interchangeably use ${\mathcal{I}}$ and $\mathcal{I}$ to denote set of temporally extended intents.}, we have
\begin{equation}
    \Big|\Big| V^*_{\cal M_{\cal I}} - V^{\pi^*_{\hat{\cal M}_\AFnoarrow}}_{\cal M_{\cal I}} \Big|\Big|_{\infty} \leq  2 \max_{\pi \in \Pi_{\mathcal{I}}} ||V^{\pi}_{\cal M_{\cal I}} - V^{\pi}_{\hat{\cal M}_\AFnoarrow}||_{\infty}.
\end{equation}
\label{lemma2smdp-trajectories}
}
\begin{proof}
$\forall s \in \mathcal{S}$,
Let us consider:
\begin{align*}
    &  V^*_{\cal M_{\cal I}}(s) - V^{\pi^*_{\hat{\cal M}_\AFnoarrow}}_{\cal M_{\cal I}}(s) \\
    &= \Big( V^*_{\cal M_{\cal I}}(s) - V^{\pi^*_{\cal M_{\cal I}}}_{\hat{\cal M}_\AFnoarrow}(s) \Big) + \underbrace{\Big( V^{\pi^*_{\cal M_{\cal I}}}_{\hat{\cal M}_\AFnoarrow}(s) - V^{*}_{\hat{\cal M}_\AFnoarrow}(s) \Big)}_{\leq 0} + \Big( V^{*}_{\hat{\cal M}_\AFnoarrow}(s) - V^{\pi^*_{\hat{\cal M}_\AFnoarrow}}_{\cal M_{\cal I}}(s) \Big)\\
    &\leq \Big( V^*_{\cal M_{\cal I}}(s) - V^{\pi^*_{\cal M_{\cal I}}}_{\hat{\cal M}_\AFnoarrow}(s) \Big) - \Big( V^{*}_{\hat{\cal M}_\AFnoarrow}(s) - V^{\pi^*_{\hat{\cal M}_\AFnoarrow}}_{\cal M_{\cal I}}(s) \Big) \\
    &\leq 2 \max_{\pi \in \Big\{ \pi^{*}_{\hat{\cal M}_\AFnoarrow}, \pi^{*}_{\cal M_{\cal I}}  \Big\} } \Big| V^{\pi}_{\cal M_{\cal I}}(s)  - V^{\pi}_{\hat{\cal M}_\AFnoarrow}(s) \Big|
\end{align*}
Taking a max over all states on both sides of the inequality and noticing that the set of all policies is a trivial super set of $\Big\{ \pi^{*}_{\hat{\cal M}_\AFnoarrow}, \pi^{*}_{\cal M_{\cal I}}  \Big\}$, we get the equation in Lemma 2 above. Moreover since, our definition of $\Pi_{\mathcal{I}}$ is a superset with the optimal policies included, we can further say the following: 
\begin{equation*}
    \Big|\Big| V^*_{\cal M_{\cal I}} - V^{\pi^*_{\hat{\cal M}_\AFnoarrow}}_{\cal M_{\cal I}} \Big|\Big|_{\infty} \leq  2 \max_{\pi \in \Pi_{\mathcal{I}}} ||V^{\pi}_{\cal M_{\cal I}} - V^{\pi}_{\hat{\cal M}_\AFnoarrow}||_{\infty}.
\end{equation*}
\end{proof}

\lemma{ For any SMDP $\hat{\cal M}_\AFnoarrow$ bounded by $[0, R^{\mathcal{O}}_{max}]$ with corresponding value function bounded by $V_{max}$ which is an approximate of the SMDP estimated from data experienced in the world for a set of intents ${\cal I}$,
\begin{equation}
    \Big|\Big|V^{\pi}_{\cal M_{\cal I}} - V^{\pi}_{\hat{\cal M}_\AFnoarrow}\Big|\Big|_{\infty} \leq \frac{1}{\Big( 1- \upgamma^{\mathcal{I}} \Big)} \max_{s, o} \Big| (\hat{R}_{I} (s,o) +  \langle \hat{\gamma}(s,o,;) \hat{P_{I}}(s,o,;), V^{\pi}_{\cal M_{\cal I}} \rangle) -  V^{\pi}_{\cal M_{\cal I}}  \Big|.
\end{equation}
\label{lemma3smdp-trajectories}
}

\begin{proof}
Given any policy over options $\pi$, define state-value function $V_0, V_1, \dots V_m$ such that $V_0 = V^{\pi}_{\cal M_{\cal I}}$,

From this point onward, we use $\AFnoarrow(o)$ and $\AFnoarrow(s)$ to denote affordable states and affordable options respectively. Recall that $\AFnoarrow \subseteq {\cal S} \times \mathcal{O}$.

$\forall s \in \AFnoarrow(o)$, 
\begin{equation*}
    V_m(s) = \sum_{o \in \AFnoarrow(s)} \pi(o|s) \Big( \hat{R}(s,o) +  \langle \hat{P_{I}}(s,o,;), V_{m-1} \rangle \Big)
\end{equation*}

Now, rewriting the above in new format:
\[
V_m(s) = \sum_{o} \pi(o|s) \Bigg[ \sum_{s'} \sum_{t=1}^\infty \sum_{\tau(s,t,s')} \hat{P_{I}}(\tau(s,t,s')|o)[ G(\tau(s,t,s')) + \gamma(\tau(s,t,s')) V_{m-1}(s')] \Bigg]
\]

Therefore:
\begin{align}
\begin{split}
    ||V_m - V_{m-1}||_{\infty} &= \max_s \left[  \sum_{o \in \AFnoarrow(s)} \pi(o|s) \sum_{s'} \sum_{t=1}^\infty \sum_{\tau(s,t,s')} \hat{P_{I}}(\tau(s,t,s')|o) \gamma(\tau(s,t,s')) (V_{m-1}(s') - V_{m-2}(s')) \right] \\
    &\leq \max_s \sum_{o \in \AFnoarrow(s)} \pi(o|s) \sum_{s'} \sum_{t=1}^\infty \sum_{\tau(s,t,s')} \hat{P_{I}}(\tau(s,t,s')|o) \gamma(\tau(s,t,s')) ||V_{m-1} - V_{m-2}||_\infty \\
    &= \max_s \sum_{o \in \AFnoarrow(s)} \pi(o|s) \sum_{s'} \gamma_o^I(s,s') ||V_{m-1} - V_{m-2}||_\infty
\end{split}
\end{align}

Since $\mathrm{E}[\sum_{s'} \gamma_o^I(s,s')]  \leq \max_{s,o} \sum_{s'} \gamma_o^I(s,s')$, therefore 
\begin{align*}
    ||V_m - V_{m-1}||_{\infty} \leq \underbrace{\max_{s,o} \sum_{s'} \gamma_o^I(s,s')}_{\upgamma^{\mathcal{I}}} || V_{m-1} - V_{m-2} ||_\infty
\end{align*}

Therefore, 
$$
||V_m - V_0||_\infty \sum_{k=0}^{m-1} ||V_{k+1} - V_k||_\infty \leq  ||V_1 - V_0||_\infty \sum_{k=1}^{m-1} (\upgamma^{\mathcal{I}})^{k-1}.$$

Taking the limit $m \rightarrow \infty $, $V_m \rightarrow V_{\hat{\cal M}_\AFnoarrow}^{\pi}$, we have:
\begin{equation*}
    ||V_{\hat{\cal M}_\AFnoarrow} - V_0||_{\infty} \leq \frac{1}{\Big( 1-\upgamma^{\mathcal{I}} \Big)} ||V_1 - V_0||_\infty
\end{equation*}
where notice that $V_0 = V^{\pi}_{\cal M_{\cal I}}$ and 
$$V_1 = \sum_{o \in \AFnoarrow(s)} \pi(o|s) \Big( \hat{R}_{I} + \langle \gamma(s,o,;) \hat{P_{I}}(s,o;), V_{M}^{\pi} \rangle \Big).$$

Therefore,  
\begin{align*}
    & \Big|\Big|V^{\pi}_{\cal M_{\cal I}} - V^{\pi}_{\hat{\cal M}_\AFnoarrow}\Big|\Big|_{\infty} \\
    & \leq \frac{1}{\Big( 1-\upgamma^{\mathcal{I}} \Big)} \max_{s} \Big| \sum_{o \in \AFnoarrow(s)} \pi(o|s) (\hat{R}_{I} (s,o) +  \langle \gamma(s,o,;) \hat{P_{I}}(s,o,;), V^{\pi}_{\cal M_{\cal I}} \rangle) -  V^{\pi}_{\cal M_{\cal I}}  \Big| \\
    & \leq \frac{1}{\Big( 1-\upgamma^{\mathcal{I}} \Big)} \max_{s,o} \Big| (\hat{R}_{I} (s,o) +  \langle \gamma(s,o,;) \hat{P_{I}}(s,o,;), V^{\pi}_{\cal M_{\cal I}} \rangle) -  V^{\pi}_{\cal M_{\cal I}}  \Big|.
\end{align*}
\end{proof}

Next, we turn to Lemma 4.

\lemma{ For any SMDP $\hat{\cal M}_\AFnoarrow$ with value function bounded by $V_{max}$ which is an approximate of the SMDP estimated from data experienced in the world for a set of intents ${\cal I}$, The following holds with probability at least $1-\delta$:
\begin{equation*}
\left\| V^{*}_{\cal M_{\cal I}} - V^{\pi^*_{\hat{\cal M}_\AFnoarrow}}_{\cal M_{\cal I}} \right\|_{\infty} \leq \frac{{2 R^{\mathcal{O}}_{max}}}{\Big( 1-\upgamma^{\mathcal{I}} \Big)\Big( 1-\upgamma^{\cal O} \Big)} \sqrt{\frac{1}{2n} \log \frac{2 |\AFnoarrow| |\Pi_{\mathcal{I}}|}{\delta}}.
\end{equation*}
\label{lemma4smdp-trajectories}}

\begin{proof}
Using Lemma~\ref{lemma2smdp-trajectories} (L2) and Lemma~\ref{lemma3smdp-trajectories}( L3), we have
\begin{equation*}
\begin{split}
    \left\| V^{\pi^*_{\cal M_{\cal I}}}_{\cal M_{\cal I}} - V^{\pi^*_{\hat{\cal M}_\AFnoarrow}}_{\cal M_{\cal I}} \right\|_{\infty} \leq  2 \max_{\pi \in \Pi_{\mathcal{I}}} \left\| V^{\pi}_{\cal M_{\cal I}} - V^{\pi}_{\hat{\cal M}_\AFnoarrow} \right\|_{\infty} \text{L2.}\\
    \leq \frac{2}{\Big( 1-\upgamma^{\mathcal{I}} \Big)} \max_{ \substack{ \pi \in \Pi_{\mathcal{I}} \\ 
    s \times o \in \AFI}}  \Big| (\hat{R}_{I} (s,o) +  \langle \gamma(s,o,;) \hat{P_{I}}(s,o,;), V^{\pi}_{\cal M_{\cal I}} \rangle) -  V^{\pi}_{\cal M_{\cal I}}  \Big| \text{L3.}
\end{split}
\end{equation*} 
Since $ (\hat{P_{I}}(s,o,;), V^{\pi}_{\cal M_{\cal I}} \rangle) -  V^{\pi}_{\cal M_{\cal I}})$ is the average of the IID samples the agent obtains by interacting with the environment, bounded in $[0, V_{max}]$ with mean $V^{\pi}_{\cal M_{\cal I}}$ (for any $s, o, \pi$ tuple i.e. state, option and policy over options tuple). Then according to Hoeffdings inequality,

\begin{equation*}
    \forall t \geq 0, \; P\Big( \Big|\sum_{o \in \AFnoarrow(s)} (\hat{R}_{I} (s,o) +  \langle \gamma(s,o,;) \hat{P_{I}}(s,o,;), V^{\pi}_{\cal M_{\cal I}} \rangle) - V^{\pi}_{\cal M_{\cal I}} \Big| > t \Big) \leq 2 \exp \left\{ \frac{-2 n t^{2}}{(V_{max})^{2}} \right\}
\end{equation*}
To obtain a uniform bound over all $s, o, \pi$ tuples, we equate the RHS to $\frac{\delta}{|\AFnoarrow(o)||\AFnoarrow(s)|\Pi_{\mathcal{I}}|}$ and the result follows as shown below.

\begin{equation*}
\begin{split}
    2 \exp \left\{ \frac{-2 n t^{2}}{(V_{max})^{2}}\right\} &= \frac{\delta}{|\AFnoarrow(o)||\AFnoarrow(s)||\Pi_{\mathcal{I}}|} \\
    \frac{-2 n t^{2}}{(V_{max})^{2}} &= \log \frac{\delta}{2|\AFnoarrow(o)||\AFnoarrow(s)||\Pi_{\mathcal{I}}|} \\
    \frac{2 n t^{2}}{(V_{max})^{2}} &= \log \frac{2|\AFnoarrow(o)||\AFnoarrow(s)||\Pi_{\mathcal{I}}|}{\delta} \\
    t^{2} &= V_{max} \frac{1}{2n} \log \frac{2 |\AFnoarrow(o)||\AFnoarrow(s)||\Pi_{\mathcal{I}}|}{\delta} \\
    t &= V_{max} \sqrt{\frac{1}{2n} \log \frac{2 |\AFnoarrow(o||\AFnoarrow(s)||\Pi_{\mathcal{I}}|}{\delta}}
\end{split}
\end{equation*}

We express the state-option pairs in affordances as the size of affordances. Formally, the size of affordances for a intent can be expressed as $|\AFnoarrow|$. Plugging this back, and using Remark~\ref{remark:generalupperboundonvstar}, we get the final result.
\end{proof}

\begin{lemma}
\label{lemma:jiangv2smdp-trajectories}
Given any policy over options $\pi$, we have 
\begin{equation}
    \Big| \Big| V^{\pi}_{\cal M} - V^{\pi}_{\cal M_{\cal I}}\Big| \Big|_\infty \leq \frac{1}{(1 - \upgamma^{\mathcal{I}})} \Big( 2\zeta^{\mathcal{I}}_R + \Big| \Big| V^{\pi}_{\cal M} \Big| \Big|_\infty  \max_{s,o} \sum_{t=1}^\infty \gamma^t  |\mathcal{S}| \zeta^{\mathcal{I}}_P \Big)
\end{equation}
\end{lemma}

\begin{proof}
We will use the following Bellman operator:
\begin{align*}
    {\cal T}^\pi_{{\cal M}}f &= \sum_{o} \pi(o|s) \Big[  \sum_{s'} \sum_{t=1}^\infty \sum_{\tau(s,t,s')} P(\tau(s,t,s')|o) [  G(\tau(s,t,s') + \gamma(\tau(s,t,s')) f(s')]  \Big] 
\end{align*}

\begin{align*}
&({\cal T}^\pi_{{\cal M}_1} - {\cal T}^\pi_{{\cal M}_2}) f(s) \\
& \quad = \sum_o \pi(o |s) \Big[ \Big(R_1(s,o) - R_2(s,o) \Big) +  \sum_{s'} \sum_{t=1}^\infty \sum_{\tau(s,t,s')} \gamma^t f(s') \Big( P_1(\tau(s,t,s')|o) - P_2(\tau(s,t,s')|o) \Big) \Big] \\
& \quad = \sum_o \pi(o |s)  \Big(R_1(s,o) - R_2(s,o) \Big) +  \sum_o \pi(o |s) \sum_{s'} \sum_{t=1}^\infty \gamma^t \sum_{\tau(s,t,s')}  f(s') \Big( P_1(\tau(s,t,s')|o) - P_2(\tau(s,t,s')|o) \Big) \\
& \quad \leq \zeta^{\mathcal{I}}_R + \Big| \Big| f\Big| \Big|_\infty  \max_{s,o} \sum_{t=1}^\infty \gamma^t \sum_{s'} \sum_{\tau(s,t,s')}  \Big( P_1(\tau(s,t,s')|o) - P_2(\tau(s,t,s')|o)\Big) \\
& \quad \leq \zeta^{\mathcal{I}}_R + \Big| \Big| f\Big| \Big|_\infty  \max_{s,o} \sum_{t=1}^\infty \gamma^t \sum_{s'} \Big[ \sum_{\tau(s,t,s')} \Big(  P_1(\tau(s,t,s')|o) -  P_2(\tau(s,t,s')|o)\Big) \Big] \\
& \quad \leq \zeta^{\mathcal{I}}_R + \Big| \Big| f\Big| \Big|_\infty  \sum_{t=1}^\infty \gamma^t  |\mathcal{S}| \zeta^{\mathcal{I}}_P
\end{align*}

and 
\begin{align*}
&{\cal T}^\pi_{\cal M} f_1(s) - {\cal T}^\pi_{\cal M} f_2(s) = \\
& \quad = \sum_o \pi(o |s)  \Big(R_1(s,o) - R_2(s,o) \Big) +  \sum_o \pi(o |s) \sum_{s'} \sum_{t=1}^\infty \sum_{\tau(s,t,s')} \gamma^t P_{\cal M}(\tau(s,t,s') \Big( f_1(s') - f_2(s') \Big)\\
& \quad \leq \zeta^{\mathcal{I}}_R  + \Big| \Big| f_1 - f_2 \Big| \Big|_\infty  \max_{s,o} \sum_{s'} \gamma_o^{\cal M}(s,s')
\end{align*}

Now, the following holds for the initial value error we are interested to bound:
\begin{align*}
& ||V^{\pi}_{\cal M} - V^{\pi}_{\cal M_{\cal I}}||_{\infty} \leq  ||V^{\pi}_{\cal M} - {\cal T}^\pi_{\cal M_{\cal I}} V^{\pi}_{\cal M}||_{\infty} + ||{\cal T}^\pi_{\cal M_{\cal I}} V^{\pi}_{\cal M} - V^{\pi}_{\cal M_{\cal I}}||_{\infty} \\
& \quad \quad = ||{\cal T}^\pi_{\cal M} V^{\pi}_{\cal M} - {\cal T}^\pi_{\cal M_{\cal I}} V^{\pi}_{\cal M}||_{\infty} + ||{\cal T}^\pi_{\cal M_{\cal I}} V^{\pi}_{\cal M} - {\cal T}^\pi_{\cal M_{\cal I}} V^{\pi}_{\cal M_{\cal I}}||_{\infty} \\
& \quad \quad = ||({\cal T}^\pi_{\cal M} - {\cal T}^\pi_{\cal M_{\cal I}}) V^{\pi}_{\cal M}||_{\infty} + ||{\cal T}^\pi_{\cal M_{\cal I}} (V^{\pi}_{\cal M}  - V^{\pi}_{\cal M_{\cal I}})||_{\infty} \\
& \quad \quad \leq
\zeta^{\mathcal{I}}_R + \Big| \Big| V^{\pi}_{\cal M} \Big| \Big|_\infty  \max_{s,o} \sum_{t=1}^\infty \gamma^t  |\mathcal{S}| \zeta^{\mathcal{I}}_P + \zeta^{\mathcal{I}}_R + \max_{s,o} \sum_{s'} \gamma_o^{I}(s,s')||V^{\pi}_{\cal M} - V^{\pi}_{\cal M_{\cal I}} ||_{\infty}
\end{align*}

Unfolding the above to infinity, we obtain in the limit the following:
$$ ||V^{\pi}_{\cal M} - V^{\pi}_{\cal M_{\cal I}}||_{\infty} \leq \frac{1}{(1 - \max_{s,o} \sum_{s'} \gamma_o^{I}(s,s'))} \Big( 2\zeta^{\mathcal{I}}_R + \Big| \Big| V^{\pi}_{\cal M} \Big| \Big|_\infty  \max_{s,o} \sum_{t=1}^\infty \gamma^t  |\mathcal{S}| \zeta^{\mathcal{I}}_P \Big) $$

Therefore,
$$ ||V^{\pi}_{\cal M} - V^{\pi}_{\cal M_{\cal I}}||_{\infty} \leq \frac{1}{(1 - \upgamma^{\mathcal{I}})} \Big( 2\zeta^{\mathcal{I}}_R + \Big| \Big| V^{\pi}_{\cal M} \Big| \Big|_\infty  \max_{s,o} \sum_{t=1}^\infty \gamma^t  |\mathcal{S}| \zeta^{\mathcal{I}}_P \Big) $$ 

\end{proof}

\textbf{Plugging Lemmas Back.}
Now the following holds for the original LHS of the planning loss bound we are after.
\begin{align*}
\Big| \Big|  V^*_{\cal M} - V^{\pi^*_{\hat{\cal M}_\AFnoarrow}}_{\cal M} \Big| \Big|_{\infty} &\leq \Big| \Big|  V^*_{\cal M} - V^{\pi^*_{\cal M_{\cal I}}}_{\cal M} \Big| \Big|_{\infty} + \Big| \Big|  V^{\pi^*_{\cal M_{\cal I}}}_{\cal M} - V^{*}_{\cal M_{\cal I}} \Big| \Big|_{\infty} + \\
&\Big| \Big| V^{*}_{\cal M_{\cal I}} - V^{\pi^*_{\hat{\cal M}_\AFnoarrow}}_{\cal M_{\cal I}} \Big| \Big|_{\infty} + \Big| \Big| V^{\pi^*_{\hat{\cal M}_\AFnoarrow}}_{\cal M_{\cal I}} - V^{\pi^*_{\hat{\cal M}_\AFnoarrow}}_{\cal M} \Big| \Big|_{\infty}
\end{align*}

Theorem~\ref{thm:trajectories_value_loss_analysis} applies to the first term, Lemma~\ref{lemma:jiangv2smdp-trajectories} to the second and forth term, and Lemma~\ref{lemma4smdp-trajectories} for the third term. Finally,

\begin{align*}
\Big| \Big|  V^*_{\cal M} - V^{\pi^*_{\hat{\cal M}_\AFnoarrow}}_{\cal M} \Big| \Big|_{\infty} & \leq \frac{1}{\Big( 1-\upgamma^{\mathcal{I}} \Big)} \zeta^{\mathcal{I}}_R + \frac{2 R^{\mathcal{O}}_{max}}{\Big( 1-\upgamma^{\mathcal{I}} \Big) \Big( 1-\upgamma^{\cal O} \Big)} \max_{s,o} \sum_{t=1}^\infty \gamma^t  |\mathcal{S}| \zeta^{\mathcal{I}}_P + \\
& \frac{2}{(1 - \upgamma^{\mathcal{I}})} \Big( 2\zeta^{\mathcal{I}}_R + \frac{R^{\mathcal{O}}_{max}}{\left( 1- \upgamma^{\cal O} \right)} \max_{s,o} \sum_{t=1}^\infty \gamma^t  |\mathcal{S}| \zeta^{\mathcal{I}}_P \Big) + \\
&\frac{{2 R^{\mathcal{O}}_{max}}}{\Big( 1-\upgamma^{\mathcal{I}}\Big) \Big( 1-\upgamma^{\cal O} \Big)} \sqrt{\frac{1}{2n} \log \frac{2 |\AFnoarrow| |\Pi_{\mathcal{I}}|}{\delta}}
\end{align*}

Rearranging terms, we get:
\begin{align*}
\Big| \Big|  V^*_{\cal M} - V^{\pi^*_{\hat{\cal M}_\AFnoarrow}}_{\cal M} \Big| \Big|_{\infty} & \leq \frac{5 \zeta^{\mathcal{I}}_R}{\Big( 1-\upgamma^{\mathcal{I}} \Big)}  + \frac{2 R^{\mathcal{O}}_{max}}{\Big( 1-\upgamma^{\mathcal{I}} \Big) \Big( 1-\upgamma^{\cal O} \Big)} \Big( 2 \max_{s,o} \sum_{t=1}^\infty \gamma^t  |\mathcal{S}| \zeta^{\mathcal{I}}_P + \sqrt{\frac{1}{2n} \log \frac{2 |\AFnoarrow| |\Pi_{\mathcal{I}}|}{\delta}} \Big)
\end{align*}

\end{proof}

\subsubsection{\textbf{Corollary 3. SMDP - Multi-Time-Model of Intent : Planning Loss}}
\label{sec:app-smdpplanningloss}

Analogous to the value loss analysis, we obtain the special case of planning loss bound for multi-time-model of an option intent as follows:

\begin{corollary} [Multi-Time-Model of Intent- Planning Loss.]
\label{corollary-smdp-planningloss}
Let $\mathcal{M}$ be any SMDP, $\TEISet$ a set of temporally extended multi-time-model of intents, $\mathcal{O}$ a set of options, and $\hat{M}_\AF$ the corresponding approximate SMDP over affordable state-option pairs $\AF$. Then, the certainty equivalence planning loss with $\hat{M}_\AF$ is
\begin{align*}
    \Big|\Big| V^*_\mathcal{M} - V^{\pi^{*}_{\hat{\mathcal{M}}_\AF}}_\mathcal{M} \Big|\Big|_{\infty} \leq \frac{2 R^{\mathcal{O}}_{max}}{(1-\upgamma^{\cal O})^2} \Bigg( 2 \upgamma \zeta^{\mathcal{I}^{\rightarrow}} +  \sqrt{\frac{1}{2n} \log \frac{2 |\AF| |\Pi_{\TEISet}|}{\delta}}   \Bigg)
\end{align*} with probability at least $1-\delta$, where $\zeta^{\mathcal{I}^{\rightarrow}}$ is the degree of satisfaction of the intents (Eq.~\ref{ass:zetaconstant}), $R^{\mathcal{O}}_{max}=\max_{s,o} r(s,o)$ is the maximum option reward, and $\upgamma^{\cal O}= \max_{s,o} \sum_{s'} \gamma_o(s,s')$ is the maximum expected discount factor for both intent and option model.
\end{corollary}

\begin{proof}
We now show that the trajectories-based planning loss bound can be reduced to the special case where intents were defined via sub-probability distributions incorporating both time and final state.

First, we consider the trajectories-based planning loss bound:
\begin{align*}
    \Big| \Big|  V^*_{\cal M} - V^{\pi^*_{\hat{\cal M}_\AFnoarrow}}_{\cal M} \Big| \Big|_{\infty} & \leq \frac{5 \zeta^{\mathcal{I}}_R}{\Big( 1-\upgamma^{\mathcal{I}} \Big)} + \frac{{2 R^{\mathcal{O}}_{max}}}{\Big( 1-\upgamma^{\mathcal{I}} \Big) \Big( 1-\upgamma^{\cal O} \Big)} \Big( 2 \max_{s,o} \sum_{t=1}^\infty \gamma^t  |\mathcal{S}| \zeta^{\mathcal{I}}_P + \sqrt{\frac{1}{2n} \log \frac{2 |\AFnoarrow| |\Pi_{\mathcal{I}}|}{\delta}} \Big)
\end{align*}

We plug our assumption that rewards are known and given which results in the constant $ \zeta^{\mathcal{I}}_R = 0 $, option and intent discount factors are assumed to be the same i.e. $\upgamma^{\cal O} = \upgamma^{\mathcal{I}} $, and the second term can be simplified further as follows:

\begin{align}
\label{eq:corollarysmdpplanningloss}
\Big| \Big|  V^*_{\cal M} - V^{\pi^*_{\hat{\cal M}_\AFnoarrow}}_{\cal M} \Big| \Big|_{\infty} & \leq \frac{{2 R^{\mathcal{O}}_{max}}}{\Big( 1-\upgamma^{\cal O} \Big)^2} \times \Big( 2 \underbrace{\max_{s,o}\sum_{\tau(s,t,s')} \gamma(\tau(s,t,s')) \Big|  P(\tau(s,t,s')|o) - P_I(\tau(s,t,s')|o) \Big|}_{ \leq \upgamma \zeta^{\mathcal{I}}} + \\ &\sqrt{\frac{1}{2n} \log \frac{2 |\AFnoarrow| |\Pi_{\mathcal{I}}|}{\delta}} \Big) \nonumber\\
&\leq \frac{{2 R^{\mathcal{O}}_{max}}}{\Big( 1-\upgamma^{\cal O} \Big)^2} \times \Big( 2 \upgamma \zeta^{\mathcal{I}} + \sqrt{\frac{1}{2n} \log \frac{2 |\AFnoarrow| |\Pi_{\mathcal{I}}|}{\delta}} \Big)
\end{align}
\end{proof}

\subsection{Intent expression on end-state}

Consider the definition of $Q^*(s,o)$ from Sec.~\ref{sec:affordances} and note that it can be re-written in our notation as:
\[
Q^*(s,o) = \sum_{s'} (r(s,o,s') + \gamma_o(s,s') \max_{o'} Q^*(s',o')) 
\]
Note that $\gamma_o(s,s') \leq \gamma$. With this notation, it is clear that the previous results from Sec.~\ref{appsec:valueloss} and Sec.~\ref{appsec:planloss} on value loss and planning loss from~\cite{khetarpal2020i} apply readily. In particular, if options only take a single step, we recover exactly their bounds, as the reward difference upper bound $\zeta^{\mathcal{I}}_R$ will be 0 and the above inequality becomes equality i.e. $\gamma_o(s,s') = \gamma$.

\section{Details of Experiments}
\label{sec:appendix:experimental_deets}

\subsubsection{Implementation Details}
\label{sec:appendix:experimental_deets:technical}
We use the environment implementation from OpenAI Gym\footnote{\href{https://github.com/openai/gym/blob/master/gym/envs/toy_text/taxi.py}{https://github.com/openai/gym/blob/master/gym/envs/toy\_text/taxi.py}}. We build upon open source code released by \cite{khetarpal2020i} significantly scaling it up using \href{https://github.com/deepmind/launchpad}{Launchpad} \citep{yang2021launchpad}. We implemented three nodes:
\begin{enumerate*}
    \item Data collection (Rollout): Runs options, $\pi_o(a|s)$, in the environment to collect transition data. 
    \item Model (and affordance) learning (Trainer): Uses the data from the Rollout node to train the option models and affordance models where relevant.
    \item Planning and evaluation (Evaluation): Uses the trained options models to perform value iteration and obtain a policy over options, $\pi_{\mathcal{O}}(o_t|s_t)$. The policy over options, $\pi_{\mathcal{O}}(o_t|s_t)$, and options, $\pi_o(a|s)$, are then evaluated over 1000 episodes to record the proportion that successfully dropped the passenger.
\end{enumerate*}

We used a shared internal cluster and each run used $\approx3$ cpus for $\approx 48$ hours. We used linear networks for all models. We initialize the affordance classifier to output 1 by shifting the input to the final sigmoid by 2, i.e. $A_\theta(s,o,s',I) = \textit{sigmoid}(f_\theta(s,o,s',I)+2)$, where $f_\theta$ is a linear model.

\begin{figure*}[t]
    \begin{center}
    \includegraphics[width=0.8\textwidth]{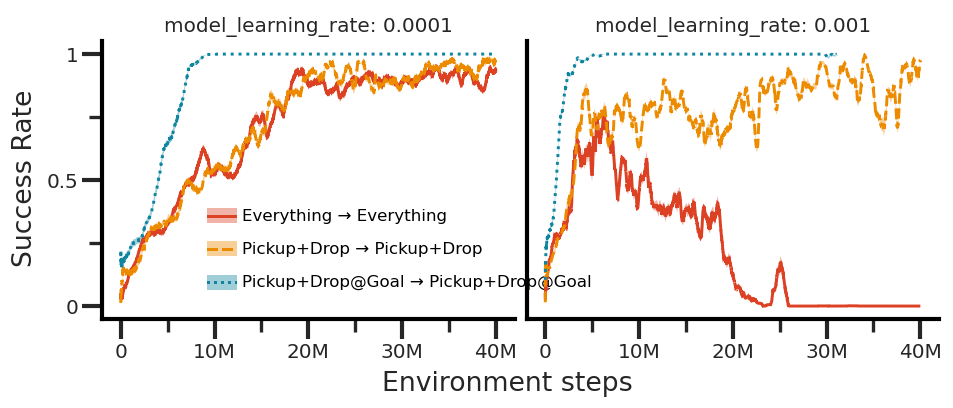}
    \end{center}
    \caption{\label{fig:heuristic_affordances_learning_rate}\textbf{
    A higher learning rate can be used to learn the model when using affordances}. Right shows divergence when using an unrestricted affordance set (Everything) for a higher learning rate compared to using any affordances.}
\end{figure*}

\subsubsection{Hyperparameter Settings}
\label{sec:appendix:experimental_deets:hyperparameters}
Given the simplicity and purpose of our experiments we only did a hyperparameter sweep over the learning rate (0.001, 0.0001).  We chose the maximum option length to be a 100 to allow options to terminate naturally. We set the hidden size of the models to be 0 (i.e. linear models). Each experiment was repeated for 4 independent seeds. We use the color-blind friendly palette from \citet{PNW_COLORBLIND} for our figures.

\section{Sample Complexity Analysis - Multi-Time-Model of Intent}
\label{sec:samplecomplexity}
Classical methods for planning in RL assume access to the complete knowledge of the MDP. However, in large domains, this is an infeasible assumption. A common approach is to consider sample-based models in which the transitions are estimated by sampling the model, with the number of calls to this sampler referred to as the sample complexity. In practise, a model $\hat{P}$ is estimated to approximate the transition model which is then used for planning (See Sec~\ref{sec:planninglossboundproof}).

We then ask the question of how difficult is to build an approximate model for everything in an environment. It is intuitive to see that modelling one-time step dynamics would require samples in the order of magnitude of the size of the state-action space (See Table~\ref{tab:samplecomplexity-comparison}). To mitigate this, we propose constructing temporally abstract partial models. Specifically, we examine the sample complexity of obtaining an $\varepsilon$ estimation of the optimal action-value function given only access to a generative model~\citep{kearns1999finite, kakade2003sample, azar2012sample}.

Consider a SMDP $\mathcal{M}$ where a deterministic policy over options is a map $\pi: \mathcal{S} \rightarrow \mathcal{O}$ that maps a state into an option. The value function of a policy $\pi$ is a vector $V^{\pi} \in \mathbb{R}^{|\mathcal{S}|}$, defined as follows, $\forall s \in \mathcal{S }$:
\begin{equation*}
   \textstyle  V^{\pi}(s) :=  \textstyle \sum_{o \in \mathcal{O}} \pi(o|s)\left[ r(s,o) + \sum_{s'} p(s' | s,o) V^{\pi}(s') \right],
\end{equation*}
where $p(s' | s,o)  = \sum_{k=1}^{\infty} p(s',k) \gamma^{k}$.

Analogously, the option value function $Q^{\pi} \in \mathbb{R}^{|\mathcal{S} \times \mathcal{O}|}$, for a policy $\pi$ is defined as follows, $\forall s \in \mathcal{S} \times \mathcal{O}$
\begin{equation*}
      Q^{\pi}(s, o) :=  r(s,o) + (P_o \cdot V^{\pi})(s, o),
\end{equation*}
where $$ P_o =\sum_{s'} p(s' | s,o), \;\; 
V^{\pi}(s')=\sum_{o' \in \mathcal{O'}} \pi(o'|s') Q^{\pi}(s', o')$$

As described earlier, we assume access to a generative model, which can provide us with samples at the SMDP level $\{ s', \tau\} \sim P(\cdot|s, o)$. Similar to previously described setting, we consider a set of {\em temporally extended intents} $\TEISet$, with the assumption that each option $o$ has an intent associated with it $I_o$, resulting in an induced SMDP $\mathcal{M}_\mathcal{I}$, with corresponding option models denoted by $P^{I}_o$. Let $\hat{\cal M}_\AF$ be the approximate SMDP over affordable state-option pairs denoted by $\AF$, with $\hat{P}^{I}_o$ as the corresponding options model.
\begin{table*}[th]
\centering
\label{tab:param-OC}
\begin{tabular}{l|l|l|}
&\multicolumn{2}{c|}{\textbf{Sample Complexity}} \\ \hline
\textbf{Actions} & \textbf{Without Affordances} & \textbf{Affordance-aware} \\ \hline
\textbf{Primitive}   & $\mathcal{O} \Big( \frac{|\mathcal{S} |\mathcal{A}|}{(1-\gamma)^4 \varepsilon^2} \Big) $   & $\mathcal{O} \Big( \frac{ |\AFI|}{(1-\gamma)^{4} {\varepsilon}^2} \Big)$      \\
\textbf{Temporally Extended}   & $\mathcal{O} \Big( \frac{|\mathcal{S}| |\mathcal{O}|}{(1-\upgamma)^4 \varepsilon^2} \Big) $   &  $\mathcal{O} \Big( \frac{ |\AFnoarrow|}{(1-\upgamma)^{4} {\varepsilon}^2} \Big)$ \\
\end{tabular}
\caption{\textbf{Comparison of Sample Complexity} - provides evidence on the role of temporal abstraction and affordances in obtaining an $\varepsilon$ estimation of the optimal value function. Incorporating affordances results in potential improvements in sample complexity in both primitive and temporally extended actions, although at the cost of approximation error induced via intents.  Here $\upgamma$ is the maximum expected discount factor for both intent and option model.}
\label{tab:samplecomplexity-comparison}
\end{table*}

We then define $\hat{P}^{I}_o$, our empirical model for each option $o \in \mathcal{O}$ be defined as follows. $\forall o \in \mathcal{O}$:
\begin{align*}
    \hat{P}_I(s'|s,o) &= \frac{\texttt{count}(s,o,s')}{N} = \frac{\sum_{i=1}^N \mathrm{1}\{ s^{'}_i=s'\}\gamma^{\tau_i}} {N}, \\
    &\texttt{where } \{ s'_i, \tau_i\} \sim P(\cdot|s,o) \forall 1 \leq i \leq N.
\end{align*}
where $\texttt{count}$ is the number of times the state-option pair $(s, o)$ pair transitions to $s'$. Note that $\mathcal{M}_\mathcal{I}$ and $\hat{\cal M}_\AF$ are equivalent to the SMDP $\cal M$ in reward\footnote{Note that here we assume the reward function is known and deterministic and therefore is identical to the true SMDP.}, except the estimated transition dynamics instead of the true transition kernel per option i.e. $P_o$. 

To derive an $\varepsilon$ optimal estimate of the optimal value function in the SMDP, we here consider the \textit{SMDP Q-value iteration (QVI)}~\citep{sutton1999between} analogous to the primitive case of Q-value iteration, but only for state-option pairs that are affordable. See~\ref{sec:SMDP-QVI} for details.

\begin{theorem}
Let $\mathcal{M}$ be a SMDP, $\TEISet$ a set of temporally extended intents corresponding to a set of options $\mathcal{O}$. If $\hat{\cal M}_\AF$ is the corresponding approximate SMDP over affordable state-option pairs $\AF$, and $Q_k$ is returned by SMDP Q-value iteration at the $k^{th}$ epoch, with inputs including the approximate SMDP as the generative model, and number of samples $m$, where 
\[
    m = \mathcal{O} \left( \frac{ |\AF|}{(1-\upgamma)^{4} {\varepsilon}^2} \right),
\]
then with probability greater than $1-\delta$, the following holds for $\varepsilon \geq \frac{2 \zeta^{\mathcal{I}^{\rightarrow}} \upgamma}{(1-\upgamma)^2}$, and for all $s$, $o$:
\[
    ||  Q_k -  Q^{*}  ||_\infty \leq \varepsilon,
\]
where $\zeta^{\mathcal{I}^{\rightarrow}}$ is the degree of satisfaction of the intents, $\upgamma$ is the maximum expected discount factor of an option, $k = \frac{\log \Big( \frac{\varepsilon (1-\upgamma)^2 - 2 \zeta^{\mathcal{I}^{\rightarrow}} \upgamma}{2(1-\upgamma)} \Big)}{\log \upgamma}$, and $Q^*$ is the optimal option value function in the underlying SMDP $\mathcal{M}$.
\label{theorem:partialmodelsamplecomplexity}
\end{theorem}

The proof is in Appendix~\ref{sec:appendix-partialmodelamplecomplexity}. The approximation error in the intended distribution $\zeta^{\mathcal{I}^{\rightarrow}}$ predominantly governs how good an estimate of the optimal option value function can be made for a given set of intents $\TEISet$. Our results suggests that we can only guarantee approximations of $Q^*$ up to the lower bound on $\varepsilon$ i.e. $\frac{2 \zeta^{\mathcal{I}^{\rightarrow}} \upgamma}{(1-\upgamma)^2}$. %

Following through the proof of Theorem~\ref{theorem:partialmodelsamplecomplexity}, it is easy to show that the number of samples $m$ required to obtain an $\varepsilon$ estimation of the optimal $Q$-value function without incorporating affordances is proportional to the size of the state-option space as shown in Theorem~\ref{theorem:smdpsamplecomplexity}.

\begin{theorem}
Let $\mathcal{M}$ be a SMDP with a set of options $\mathcal{O}$. If $\hat{\cal M}$ is the corresponding approximate SMDP, and $Q_k$ is returned by SMDP Q-value iteration at the $k^{th}$ epoch, with inputs including the approximate SMDP as the generative model, and number of samples $m$, where 
\begin{align*}
    m = \mathcal{O} \Big( \frac{|\mathcal{S}| |\mathcal{O}|}{(1-\upgamma)^4 \varepsilon^2} \Big), 
\end{align*} then with probability greater than $1-\delta$, the following holds for all $s$ and $o$: 
\begin{equation*}
\textstyle   || Q_k -  Q^{*}  ||_\infty \leq \varepsilon,
\end{equation*}
where $\upgamma$ is the maximum expected option discount factor, $k = \frac{\log(\varepsilon (1-\upgamma))}{\log \upgamma}$, and $Q^*$ is the optimal option value function in the underlying SMDP $\mathcal{M}$.
\label{theorem:smdpsamplecomplexity}
\end{theorem}

For a complete proof, See Appendix~\ref{sec:appendix-smdpsamplecomplexity}. To summarize, Table~\ref{tab:samplecomplexity-comparison} decouples the role of temporal abstraction and the effect of incorporating affordances. Predicting and reasoning across multiple timescales naturally results in a growing set of action choices leading to a large number of samples. Larger gains can be established when considering both temporal abstractions and affordance information, with a carefully designed set of intents.

\subsection{Proofs - Sample Complexity Analysis}
\label{sec:appendix-samplecomplexity}

\textbf{Note:} We again overload notation and throughout our proofs, for convenience we interchangeably use ${\mathcal{I}}$ and $\TEISet$ to denote set of temporally extended intents. Similarly, for convenience we interchangeably use $I$ and $\TEI$ to denote a temporally extended intent for an option $o$.

\subsubsection{SMDP Q-Value Iteration (QVI)}
\label{sec:SMDP-QVI}
To derive an $\varepsilon$ optimal estimate of the optimal option-value function in the SMDP, we here consider the \textit{SMDP Q-value iteration} (SMDP-QVI)~\citep{sutton1999between} process as detailed in algorithm below.

\begin{algorithm}[h]
\begin{algorithmic}[1]
\STATE $V_0 = 0$, $Q_0 = 0$
\FOR{epoch $k=1 \dots K$ }
\FOR{$(s,o) \in \AFnoarrow$, }
\STATE $Q_k(s,o) = r(s,o) + (\hat{P}^{I}_o V_{k-1})(s,o)$
\STATE $V_k(s) = \max_{o \in \AFnoarrow(s)} Q_k(s,o)$
\ENDFOR
\ENDFOR
\STATE Output $Q_k$
\end{algorithmic}
\caption{{\bf Model-based SMDP Q-Value Iteration (SMDP-QVI)}}
\label{alg:SMDP-QVI}
\end{algorithm}

\subsubsection{\textbf{Proof of Theorem~\ref{theorem:partialmodelsamplecomplexity}} - Sample complexity of Temporally Abstract Partial Model.}
\label{sec:appendix-partialmodelamplecomplexity}

\begin{proof}
We here consider the transition models in the ground SMDP $\mathcal{M}$, the intent induced SMDP $\mathcal{M}_\mathcal{I}$, and the approximate SMDP $\hat{\cal M}_\AFnoarrow$ over affordable state-option pairs are denoted by $P_o$, $P^{I}_o$, and $\hat{P}^{I}_o$ respectively.

We here consider $\Big| \Big|  Q_k -  Q^{*}  \Big| \Big|_\infty$.

Adding and subtracting $\hat{Q}^{*}_{\hat{\cal M}_\AFnoarrow}$ and $Q^{*}_{\mathcal{M}_\mathcal{I}}$ we get,
\begin{align*}
      Q_k -  Q^{*} = \underbrace{Q_k - \hat{Q}^{*}_{\hat{\cal M}_\AFnoarrow}}_\text{Term (A)} +  \underbrace{ \hat{Q}^{*}_{\hat{\cal M}_\AFnoarrow} - Q^{*}_{\mathcal{M}_\mathcal{I}}}_\text{Term (B)} +  \underbrace{Q^{*}_{\mathcal{M}_\mathcal{I}} - Q^{*}}_\text{Term (C)}
\end{align*}

\textbf{Bounding Term (A)}
\begin{align*}
    \Big| \Big|  Q_k -  \hat{Q}^{*}_{\hat{\cal M}_\AFI}  \Big| \Big|_\infty &= \max_{(s,o) \in \AFI} \Big[ r(s,o) + (\hat{P}^{I}_o V_{k-1})(s,o) - ( r(s,o) + (\hat{P}^{I}_o \hat{V}^{*})(s,o) ) \Big] \\
    &= \max_{(s,o) \in \AFI} \Big|(\hat{P}^{I}_o ( V_{k-1} - \hat{V}^{*} )) (s,o) \Big| \\
    &\leq \upgamma \Big| \Big| V_{k-1} - \hat{V}^{*}  \Big|\Big|_\infty \\
    &\leq \upgamma \max_{s \in \AFI(o)} \Big|  \max_{o \in \AFI(s)}  Q_{k-1}(s,o) -  \max_{o \in \AFI(s)}  \hat{Q}^{*}_{\hat{\cal M}_\AFI}(s,o) \Big| \\
    &\leq \upgamma \max_{(s,o) \in \AFI} \Big|   Q_{k-1}(s,o) -  \hat{Q}^{*}_{\hat{\cal M}_\AFI}(s,o) \Big| \\
    &= \upgamma \Big| \Big| Q_{k-1} - \hat{Q}^{*}_{\hat{\cal M}_\AFI}  \Big|\Big|_\infty
\end{align*}

Unrolling the above $k$ times, we get;
\begin{align*}
    \Big| \Big|  Q_k -  \hat{Q}^{*}_{\hat{\cal M}_\AFI}  \Big| \Big|_\infty \leq (\upgamma)^k \Big| \Big| Q_{0} - \hat{Q}^{*}_{\hat{\cal M}_\AFI} \Big|\Big|_\infty \leq \frac{(\upgamma)^k }{(1-\upgamma)}
\end{align*}

\textbf{Bounding Term (B)}
\begin{align*}
    \Big( \hat{Q}^{*}_{\hat{\cal M}_\AFI} - Q^{*}_{\mathcal{M}_\mathcal{I}}  \Big) (s,o) &= (\hat{P}^{I}_o \hat{V}^{*})(s,o) - (P^{I}_o V^{*})(s,o) \\
    &= \underbrace{(\hat{P}^{I}_o V^{*} -  P^{I}_o V^{*})(s,o)} +  \underbrace{(\hat{P}^{I}_o \hat{V}^{*})(s,o) - (\hat{P}^{I}_o V^{*})(s,o)}    \text{ Adding and Subtracting $\hat{P}^{I}_o V^{*}$} \\
    &= \Big(\Big( \hat{P}^{I}_o - P^{I}_o \Big) V^{*}\Big)(s,o) - \Big(\ \hat{P}^{I}_o \Big( V^{*} - \hat{V}^{*} \Big) \Big) (s,o)\\
    &= \Big( \Big( \hat{P}^{I}_o - P^{I}_o \Big) V^{*} \Big)(s,o) - \\
    &\sum_{s' \in \AFI(o)} \hat{P}^{I}_o(s'|s,o)  \Big( \max_{o' \in \AFI(s)} Q^{*}_{\mathcal{M}_\mathcal{I}}(s',o') - \max_{o' \in \AFI(s)} \hat{Q}^{*}_{\hat{\cal M}_\AFI}(s',o')\Big)
\end{align*}

Considering the max over all state-options, we have;
\begin{align*}
    \Big| \Big|  \hat{Q}^{*}_{\hat{\cal M}_\AFI} - Q^{*}_{\mathcal{M}_\mathcal{I}}  \Big| \Big|_\infty \leq  \Big|\Big| \Big( \hat{P}^{I}_o - P^{I}_o \Big) V^{*} \Big|\Big| + \upgamma \Big| \Big|  \hat{Q}^{*}_{\hat{\cal M}_\AFI} - Q^{*}_{\mathcal{M}_\mathcal{I}}  \Big| \Big|_\infty
\end{align*}

Finally;
\begin{align*}
    \Big| \Big|  \hat{Q}^{*}_{\hat{\cal M}_\AFI} - Q^{*}_{\mathcal{M}_\mathcal{I}}  \Big| \Big|_\infty &\leq  \frac{1}{(1-\upgamma)} \Big|\Big| \Big( \hat{P}^{I}_o - P^{I}_o \Big) V^{*}_{\mathcal{M}_\mathcal{I}} \Big|\Big|
\end{align*}

Now let's fix a state option pair $(s,o) \in \AFI$
\begin{align*}
    \Big( \hat{P}^{I}_o - P^{I}_o \Big) V^{*}_{\mathcal{M}_\mathcal{I}} &= \frac{1}{N} \sum_{i=1}^N V^{*}_{\mathcal{M}_\mathcal{I}}(s'_{i}) - \mathrm{E}_{s' \in P^{I}_o(\cdot|s, o)} [V^{*}_{\mathcal{M}_\mathcal{I}}(s')] \\
    &= \frac{1}{N} \Big( S_N - \mathrm{E}[S_N] \Big)
\end{align*}

where $S_N = \sum_{i=1}^N X_i$ and $X_i = V^{*}(s^{'}_{i})$, $X_i$ are independent variable and $|X_i| \leq V_{max}$.

We now consider the Hoeffdings inequality:
\begin{align*}
        P \Big( \frac{1}{N}  (S_N - \mathrm{E}[S_N]) \geq t \Big)  \leq 2 \exp \Big( \frac{-N^2 t^2}{N V^2_{max}} \Big) = 2 \exp \Big( \frac{-Nt^2}{V^2_{max}} \Big)
\end{align*}

Applying Hoeffdings, we get;
\begin{align*}
        P \Big( \max_{s,o \in \AFI} \Big| ( \hat{P}^{I}_o - P^{I}_o ) V^{*}_{\mathcal{M}_\mathcal{I}} (s,o)\Big| \geq t \Big)  &= P \Big( \exists (s,o \in \AFI) s.t. \Big| ( \hat{P}^{I}_o - P^{I}_o ) V^{*}_{\mathcal{M}_\mathcal{I}} (s,o)\Big|  \geq t \Big)\\
        &\leq \sum_{\AFI} Pr \Big( \Big| ( \hat{P}^{I}_o - P^{I}_o \Big) V^{*}_{\mathcal{M}_\mathcal{I}} (s,o)\Big| \geq t \Big) \text{// Union Bound} \\
        &= 2|\AFI(o)| |\AFI(s)| \exp \Big( \frac{-Nt^2}{V^2_{max}} \Big) \\
        &= 2 |\AFI| \exp \Big( \frac{-Nt^2}{V^2_{max}} \Big)
\end{align*}

We assume that the failure probability $\delta \geq 0$, We then solve for $t$ by equating the RHS to $\delta$ as follows:
\begin{align*}
    2 |\AFI| \exp \Big( \frac{-Nt^2}{V^2_{max}} \Big) &= \delta \\
    \exp \Big( \frac{-Nt^2}{V^2_{max}} \Big) &= \frac{\delta}{2 |\AFI|} \\
   \frac{-Nt^2}{V^2_{max}} &= \log \frac{\delta}{2 |\AFI|} \\
   t^2 &=  \frac{V^2_{max}}{N} \log  \frac{2 |\AFI|}{\delta} \\
   t &= V_{max} \sqrt{\frac{1}{N} \log  \frac{2 |\AFI|}{\delta}}
\end{align*}

Plugging this back in Term (B) $ \Big| \Big|  \hat{Q}^{*}_{\hat{\cal M}_\AFI} - Q^{*}_{\mathcal{M}_\mathcal{I}}  \Big| \Big|_\infty \leq  \frac{1}{(1-\upgamma)} \Big|\Big| \Big( \hat{P}_o - P_o \Big) V^{*} \Big|\Big|$, we get:
\begin{align*}
   \Big| \Big|  \hat{Q}^{*}_{\hat{\cal M}_\AFI} - Q^{*}_{\mathcal{M}_\mathcal{I}}  \Big| \Big|_\infty &\leq \frac{V_{max}}{(1-\upgamma)}   \sqrt{\frac{1}{N} \log  \frac{2 |\AFI| }{\delta}}
\end{align*}

Based on Remark~\ref{remark:smdpvmaxupperbound}, 
\begin{align*}
   \Big| \Big|  \hat{Q}^{*}_{\hat{\cal M}_\AFI} - Q^{*}_{\mathcal{M}_\mathcal{I}}  \Big| \Big|_\infty &\leq \frac{R^{\mathcal{O}}_{max}}{(1-\upgamma)^2}   \sqrt{\frac{1}{N} \log  \frac{2 |\AFI| }{\delta}}
\end{align*}

\textbf{Bounding Term (C)} $\Big| \Big| Q^{*}_{\mathcal{M}_\mathcal{I}} - Q^{*} \Big| \Big|_\infty$

We first define the following optimality bellman operator:
\begin{align*}
    Q^{*}_\mathcal{M} &= \mathcal{T} Q^{*}_\mathcal{M} \\
    \text{where} \Big( \mathcal{T}f \Big) &:= R(s,o) + \langle P(s,o), V_f \rangle \\
    \text{where} V_f(\cdot) &:= \max_{o \in \mathcal{O}} f(\cdot, o)
\end{align*}

Our aim here is to bound $\Big| \Big| Q^*_{M_1} - Q^*_{ M_2} \Big| \Big|_{\infty}$ for any two SMDP models $M_1$ and $M_2$. 

Let ${\cal T}_{1}$ and ${\cal T}_{2}$ be the Bellman operator of the SMDPs $M_1$ and $M_2$ respectively. Therefore,
\begin{align*}
    \Big| \Big| Q^*_{M_1} - {\cal T}_{2} Q^*_{ M_1} \Big| \Big|_{\infty} &=  \Big| \Big| {\cal T}_{1} Q^*_{ M_1} - {\cal T}_{2} Q^*_{M_1} \Big| \Big|_{\infty} \\
    &= \max_{(s, o) \in S \times \mathcal{O}} \Big| \langle P_1(s,o), V^*_{ M_1} \rangle  -  \langle P_2(s,o), V^*_{ M_1} \rangle      \Big| \\
    &= \max_{(s, o) \in S \times \mathcal{O}} \Big| \mathbb{E}_{s' \sim P_1(s,o)} [V^*_{M_1}(s')] - \mathbb{E}_{s' \sim P_2(s,o)} [V^*_{M_1}(s')]\Big| \\
    &\leq \Big |\Big |d^{\mathrm{F}}_{M_1, M_2}\Big |\Big |_{\infty}
\end{align*}

Therefore, 
\begin{align*}
    \Big| \Big| Q^*_{M_1} - Q^*_{ M_2} \Big| \Big|_{\infty} &=  \Big| \Big| Q^*_{M_1} - {\cal T}_{2} Q^*_{ M_1} + {\cal T}_{2} Q^*_{ M_1} -  {\cal T}_{2} Q^*_{ M_2} \Big| \Big|_{\infty} \\
    &\leq \Big |\Big |d^{\mathrm{F}}_{M_1, M_2}\Big |\Big |_{\infty} + \Big| \Big| {\cal T}_{2} Q^*_{ M_1} -  {\cal T}_{2} Q^*_{ M_2} \Big| \Big|_{\infty}
\end{align*}

Bounding the second term of the last step i.e. $\Big| \Big| {\cal T}_{2} Q^*_{ M_1} -  {\cal T}_{2} Q^*_{ M_2} \Big| \Big|_{\infty}$;
\begin{align*}
    \Big| {\cal T}_{2} f_1(s,o) -  {\cal T}_{2} f_2(s,o)  \Big|  &= \Big | \big( r(s,o) + \langle P_2(s,o) V_{f_1}(s) \rangle \big) -  \big( r(s,o) +  \langle P_2(s,o) V_{f_2}(s) \rangle       \big) \Big | \\
    &= \Big | \langle P_2(s,o) V_{f_1}(s) \rangle \big) -  \langle P_2(s,o) V_{f_2}(s) \rangle  \Big | \\
    &\leq \max_{ (s,o) \in \mathcal{S} \times \mathcal{O}} \Big| \mathbb{E}_{s' \sim P_2(s,o)} [V_{f_1}(s')] - \mathbb{E}_{s' \sim P_2(s,o)} [V_{f_2}(s')]\Big| \\
    &= \max_{ (s,o) \in \mathcal{S} \times \mathcal{O}} \sum_{s'} P_2(s'|s,o)  \Big| V_{f_1}(s') - V_{f_2}(s')\Big| \\
    &\leq \upgamma \Big| \Big|  V^*_{M_1} - V^*_{ M_2}   \Big| \Big|_{\infty}
\end{align*}

Therefore,
\begin{align*}
    \Big| \Big| Q^{*}_{\mathcal{M}_\mathcal{I}} - Q^{*} \Big| \Big|_\infty \leq \Big |\Big |d^{\mathrm{F}}_{M_1, M_2}\Big |\Big |_{\infty} + \upgamma \Big| \Big|  V^*_{\mathcal{M}_\mathcal{I}} - V^*   \Big| \Big|_{\infty}
\end{align*}

where note that the second term in the last step is bounded as following, 
\begin{align*}
    \max_{s} \Big|  V^*_{M_1} - V^*_{ M_2}   \Big|  &= \max_s \Big|  \max_{o} Q^*_{M_1}(s,o) - \max_{o} Q^*_{ M_2}(s,o)   \Big| \\
    &\leq \max_s \Big|  \max_{o} ( Q^*_{M_1}(s,o) - Q^*_{ M_2}(s,o))   \Big| \\
    &\leq \max_{s,o} \Big|  Q^*_{M_1}(s,o) - Q^*_{ M_2}(s,o)  \Big| \\
    &= \Big| \Big| Q^*_{M_1} - Q^*_{ M_2} \Big| \Big|_{\infty}
\end{align*}

Therefore,
\begin{align*}
    \Big| \Big| Q^{*}_{\mathcal{M}_\mathcal{I}} - Q^{*} \Big| \Big|_\infty &\leq \Big |\Big |d^{\mathrm{F}}_{M_1, M_2}\Big |\Big |_{\infty} + \upgamma \Big| \Big|  V^*_{\mathcal{M}_\mathcal{I}} - V^*   \Big| \Big|_{\infty} \\
    &\leq \Big |\Big |d^{\mathrm{F}}_{M_1, M_2}\Big |\Big |_{\infty} + \upgamma \Big| \Big|  Q^*_{\mathcal{M}_\mathcal{I}} - Q^*   \Big| \Big|_{\infty} \\
    &\leq \frac{1}{(1-\upgamma)} \Big |\Big |d^{\mathrm{F}}_{{\mathcal{M}_\mathcal{I}}, \mathcal{M}}\Big |\Big |_{\infty} \\
    &\leq \frac{ \zeta^{\mathcal{I}}  \upgamma R^{\mathcal{O}}_{max}}{(1-\upgamma)^2}.
\end{align*}

We conclude, 
\begin{align*}
     \Big| \Big| Q_k -  Q^{*} \Big| \Big|_\infty &\leq \Big| \Big| Q_k -  \hat{Q}^{*}_{\hat{\cal M}_\AF} \Big| \Big|_\infty +  \Big| \Big| \hat{Q}^{*}_{\hat{\cal M}_\AF} - Q^{*}_{\mathcal{M}_\mathcal{I}} \Big| \Big|_\infty +  \Big| \Big| Q^{*}_{\mathcal{M}_\mathcal{I}} - Q^{*} \Big| \Big|_\infty \\
     &\leq \frac{(\upgamma)^k }{(1-\upgamma)} + \frac{1}{(1-\upgamma)^2}   \sqrt{\frac{1}{N} \log  (2 |\AFI|)} +\frac{ \zeta^{\mathcal{I}}  \upgamma R^{\mathcal{O}}_{max}}{(1-\upgamma)^2}
\end{align*}

To obtain an $\varepsilon$ estimation of the optimal $Q$-value function in the SMDP, we distribute the error across Term A, B, and C such that ;
\begin{align*}
    \Big| \Big| Q_k -  Q^{*}  \Big| \Big|_\infty &\leq \underbrace{\text{Term (A)} + \text{Term (C)}}_{\leq \varepsilon/2} + \underbrace{\text{Term (B)}}_{\leq \varepsilon/2}
\end{align*}

By choosing $k = \frac{\log \Big( \frac{\varepsilon (1-\upgamma)^2 - 2 \zeta \upgamma}{2(1-\upgamma)} \Big)}{\log \upgamma}$ and $N=\frac{4}{(1-\upgamma)^{4} \varepsilon^2} \log( 2|\AF|)$, we get $\Big| \Big| Q_k -  Q^{*}  \Big| \Big|_\infty \leq \varepsilon/2 + \varepsilon/2$

Note that this choice of $k$ holds if and only if:
\begin{align*}
    \varepsilon (1-\upgamma)^2  &\geq 2 \zeta^{\mathcal{I}} \upgamma \\
    \varepsilon &\geq \frac{2 \zeta^{\mathcal{I}} \upgamma}{(1-\upgamma)^2}
\end{align*}

Therefore, the total number of samples needed to get an $\varepsilon$-estimation of the optimal option value function is;
\begin{align*} 
    N  |\mathcal{S}| |\mathcal{O}| = \mathcal{O} \Big( \frac{ |\AF|}{(1-\upgamma)^{4} {\varepsilon}^2} \Big)
\end{align*}
\end{proof}

\subsubsection{\textbf{Proof of Theorem~\ref{theorem:smdpsamplecomplexity}} - Sample complexity of Temporally Abstract Full Model.}
\label{sec:appendix-smdpsamplecomplexity}

\begin{proof}
We here consider $\Big| \Big|  Q_k -  Q^{*}  \Big| \Big|_\infty$, and $Q^*$ is the optimal option value function in the underlying SMDP $\mathcal{M}$.

Adding and subtracting $\hat{Q}^{*}$ we get,
\begin{align*}
      Q_k -  Q* = \underbrace{Q_k - \hat{Q}^*}_\text{Term (A)} +  \underbrace{\hat{Q}^* - Q^*}_\text{Term (B)}
\end{align*}

\textbf{Bounding Term (A)}
\begin{align*}
    \Big| \Big|  Q_k -  \hat{Q}^*  \Big| \Big|_\infty &= \max_{(s,o) \in \mathcal{S} \times \mathcal{O}} \Big[ r(s,o) + \hat{P}_o V_{k-1}(s,o) - ( r(s,o) + \hat{P}_o \hat{V}^{*}(s,o) ) \Big] \\
    &= \max_{(s,o) \in \mathcal{S} \times \mathcal{O}} \Big|\hat{P}_o ( V_{k-1} - \hat{V}^{*} ) (s,o) \Big| \\
    &\leq \gamma^{\mathcal{D}} \Big| \Big| V_{k-1} - \hat{V}^{*}  \Big|\Big|_\infty \\
    &\leq \gamma^{\mathcal{D}} \max_{s \in \mathcal{S}} \Big|  \max_{o \in \mathcal{O}}  Q_{k-1}(s,o) -  \max_{o \in \mathcal{O}}  \hat{Q}^{*}(s,o) \Big| \\
    &\leq \gamma^{\mathcal{D}} \max_{(s,o) \in \mathcal{S} \times \mathcal{O}} \Big|   Q_{k-1}(s,o) -  \hat{Q}^{*}(s,o) \Big| \\
    &= \gamma^{\mathcal{D}} \Big| \Big| Q_{k-1} - \hat{Q}^{*}  \Big|\Big|_\infty
\end{align*}

Unrolling the above $k$ times, we get;
\begin{align*}
    \Big| \Big|  Q_k -  \hat{Q}^*  \Big| \Big|_\infty \leq (\gamma^{\mathcal{D}})^k \Big| \Big| Q_{0} - \hat{Q}^{*} \Big|\Big|_\infty \leq \frac{(\gamma^{\mathcal{D}})^k }{(1-\gamma^{\mathcal{D}})}
\end{align*}

\textbf{Bounding Term (B)}
\begin{align*}
    \Big( \hat{Q}^{*} - Q^{*}  \Big) (s,o) &= \hat{P}_o \hat{V}^{*}(s,o) - P_o V^{*}(s,o) \\
    &= \hat{P}_o V^{*}(s,o) - P_o V^{*}(s,o) - \hat{P}_o \hat{V}^{*}(s,o)  - \hat{P}_o V^{*}(s,o)\text{ Adding and Subtracting $\hat{P}_o V^{*}$} \\
    &= \Big( \hat{P}_o - P_o \Big) V^{*}(s,o) - \hat{P}_o  \Big( \hat{V}^{*} - V^{*} \Big) (s,o)\\
    &= \Big( \hat{P}_o - P_o \Big) V^{*}(s,o) - \sum_{s' \in \mathcal{S}} \hat{P}_o(s'|s,o)  \Big( \max_{o' \in \mathcal{O}} \hat{Q}^{*}(s',o') - \max_{o' \in \mathcal{O}} Q^{*}(s',o') \Big)
\end{align*}

Therefore;
\begin{align*}
    \Big| \Big|  \hat{Q}^{*} - Q^{*}  \Big| \Big|_\infty \leq  \Big|\Big| \Big( \hat{P}_o - P_o \Big) V^{*} \Big|\Big| + \gamma^{\mathcal{D}} \Big|\Big| \hat{Q}^{*} - Q^{*}  \Big| \Big|_\infty 
\end{align*}

Finally;
\begin{align*}
    \Big| \Big|  \hat{Q}^{*} - Q^{*}  \Big| \Big|_\infty &\leq  \frac{1}{(1-\gamma^{\mathcal{D}})} \Big|\Big| \Big( \hat{P}_o - P_o \Big) V^{*} \Big|\Big|
\end{align*}

Now let's fix a state option pair $(s,o) \in \mathcal{S} \times \mathcal{O}$
\begin{align*}
    \Big( \hat{P}_o - P_o \Big) V^{*} &= \frac{1}{N} \sum_{i=1}^N V^{*}(s'_{i}) - \mathrm{E}_{s' \in P_o(\cdot|s, o)} [V^{*}(s')] \\
    &= \frac{1}{N} \Big( S_N - \mathrm{E}[S_N] \Big)
\end{align*}

where $S_N = \sum_{i=1}^N X_i$ and $X_i = V^{*}(s^{'}_{i})$, $X_i$ are independent variable and $|X_i| \leq V_{max}$.

We now consider the Hoeffdings inequality:
\begin{align*}
        P \Big( \frac{1}{N}  (S_N - \mathrm{E}[S_N]) \geq t \Big)  \leq 2 \exp \Big( \frac{-N^2 t^2}{N V^2_{max}} \Big) = 2 \exp \Big( \frac{-Nt^2}{V^2_{max}} \Big)
\end{align*}

Applying Hoeffdings, we get;
\begin{align*}
        P \Big( \max_{\mathcal{S}, \mathcal{O}} \Big| ( \hat{P}_o - P_o ) V^{*} (s,o)\Big| \geq t \Big)  &= P \Big( \exists (s,o) s.t. \Big| ( \hat{P}_o - P_o ) V^{*} (s,o)\Big|  \geq t \Big)\\
        &\leq \sum_{\mathcal{S}, \mathcal{O}} Pr \Big( \Big| ( \hat{P}_o - P_o \Big) V^{*} (s,o)\Big| \geq t \Big) \text{// Union Bound} \\
        &= 2|\mathcal{S}| |\mathcal{O}| \exp \Big( \frac{-Nt^2}{V^2_{max}} \Big)
\end{align*}

We assume that the failure probability $\delta \geq 0$, We then solve for $t$ by equating the RHS to $\delta$ as follows:
\begin{align*}
    2|\mathcal{S}| |\mathcal{O}| \exp \Big( \frac{-Nt^2}{V^2_{max}} \Big) &= \delta \\
    \exp \Big( \frac{-Nt^2}{V^2_{max}} \Big) &= \frac{\delta}{2 |\mathcal{S}| |\mathcal{O}|} \\
   \frac{-Nt^2}{V^2_{max}} &= \log \frac{\delta}{2|\mathcal{S}| |\mathcal{O}|} \\
   t^2 &=  \frac{V^2_{max}}{N} \log  \frac{2|\mathcal{S}| |\mathcal{O}|}{\delta} \\
   t = V_{max} \sqrt{\frac{1}{N} \log  \frac{2|\mathcal{S}| |\mathcal{O}|}{\delta}}
\end{align*}

Plugging this back in Term (B) $ \Big| \Big|  \hat{Q}^{*} - Q^{*}  \Big| \Big|_\infty \leq  \frac{1}{(1-\gamma^{\mathcal{D}})} \Big|\Big| \Big( \hat{P}_o - P_o \Big) V^{*} \Big|\Big|$, we get:
\begin{align*}
   \Big| \Big|  \hat{Q}^{*} - Q^{*}  \Big| \Big|_\infty &\leq \frac{V_{max}}{(1-\gamma^{\mathcal{D}})}   \sqrt{\frac{1}{N} \log  \frac{2|\mathcal{S}| |\mathcal{O}|}{\delta}}
\end{align*}

Therefore;
\begin{align*}
    \Big| \Big| Q_k -  Q^{*}  \Big| \Big|_\infty &\leq  \Big| \Big| Q_k - \hat{Q}^{*} \Big| \Big|_\infty  +  \Big| \Big|  \hat{Q}^{*}  - Q^{*}  \Big| \Big|_\infty  \\
   &\leq \frac{(\gamma^{\mathcal{D}})^k }{(1-\gamma^{\mathcal{D}})} + \frac{V_{max}}{(1-\gamma^{\mathcal{D}})}  \sqrt{\frac{1}{N} \log  \frac{2|\mathcal{S}| |\mathcal{O}|}{\delta}} \\
   &\leq \frac{(\gamma^{\mathcal{D}})^k }{(1-\gamma^{\mathcal{D}})} + \frac{R_{max}}{(1-\gamma^{\mathcal{D}})^2}  \sqrt{\frac{1}{N} \log  \frac{2|\mathcal{S}| |\mathcal{O}|}{\delta}}
\end{align*}

To obtain an $\varepsilon$ estimation of the optimal $Q$-value function in the SMDP, we distribute the error uniformly;
\begin{align*}
    \Big| \Big| Q_k -  Q^{*}  \Big| \Big|_\infty &\leq  \varepsilon/2 + \varepsilon/2
\end{align*}

Equating each term to $\varepsilon/2$ and solving for $k$ and $N$ results in $k = \frac{\log(\varepsilon (1-\gamma^{\mathcal{D}}))}{\log \gamma^{\mathcal{D}}}$ and $N=\frac{4}{(1-\gamma^{\mathcal{D}})^{4} \varepsilon^2} \log( 2|\mathcal{S}| |\mathcal{O}|)$
Therefore;
\begin{align*} 
    N  |\mathcal{S}| |\mathcal{O}| = \mathcal{O} \Big( \frac{ |\mathcal{S}| |\mathcal{O}|}{(1-\gamma^{\mathcal{D}})^{4} {\varepsilon}^2} \Big)
\end{align*}
\end{proof}

\end{document}